%% file: sample.tex
\documentclass[twoside,11pt]{article}
\usepackage{changepage}
\usepackage{blindtext}
\usepackage{amsthm,apxproof}
\usepackage{booktabs} 
\usepackage{multirow}
\usepackage{thmtools,amsmath,amsfonts} 
\usepackage{amssymb,mathrsfs}
\usepackage{algorithm}
\usepackage{algpseudocode}
\usepackage{amssymb}
\usepackage{enumitem}
\newtheorem{rem}{Remark}

\newtheorem{thm}{Theorem}[section]
\newtheorem{lem}{Lemma}[section]

\newtheorem{assumpt}{Assumption}[section]
\newtheorem{counter}{Counterexample}[section]
\newtheorem{setting}{Setting}

\input{math_commands.tex}

\usepackage[T1]{fontenc}
\usepackage{url}
\usepackage {color}
\allowdisplaybreaks

\DeclareMathOperator{\Expect}{\mathbb{E}}

\newcommand{\I}{\mathbb{I}}
\DeclareMathOperator{\Pro}{\mathbb{P}}

\usepackage{jmlr2e}

\usepackage{lastpage}

\ShortHeadings{Sample JMLR Paper}{One and Two}
\firstpageno{1}

\begin{document}

\title{Stochastic Gradient Descent in Non-Convex Problems: Asymptotic Convergence with Relaxed Step-Size via Stopping Time Methods}

\author{\name Ruinan~Jin \email jrnjrnjrnjrnjrn@126.com \\
       \addr Centre for Artificial Intelligence and Robotics, Hong Kong Institute of Science and Innovation\\ Chinese Academy of Sciences\\
       New Territories, Hong Kong
       \AND
       \name Difei~Cheng\thanks{Corresponding author} \email chengdifei@amss.ac.cn \\
       \addr Institute of Automation\\
       Chinese Academy of Sciences\\
       Beijing, 100190, China
       \AND
       \name Hong~Qiao \email hong.qiao@ia.ac.cn \\
       \addr Institute of Automation\\
       Chinese Academy of Sciences\\
       Beijing, 100190, China
       \AND
       \name Xin~Shi \email xinshi553@gmail.com \\
       \addr Shanghai Jiao Tong University\\
       Shanghai, 200240, China
       \AND
       \name Shaodong~Liu \email liushaodong321@sjtu.edu.cn \\
       \addr Shanghai Jiao Tong University\\
       Shanghai, 200240, China
       \AND
       \name  Bo~Zhang \email b.zhang@amt.ac.cn \\
       \addr Academy of Mathematics and Systems Science\\
       Chinese Academy of Sciences\\
       Beijing, 100190, China
       }

\editor{My editor}

\maketitle

\begin{abstract}
Stochastic Gradient Descent (SGD) is widely used in machine learning research. Previous convergence analyses of SGD under the vanishing step-size setting typically require Robbins–Monro conditions. However, in practice, a wider variety of step-size schemes are frequently employed, yet existing convergence results remain limited and often rely on strong assumptions. This paper bridges this gap by introducing a novel analytical framework based on a stopping-time method, enabling asymptotic convergence analysis of SGD under more relaxed step-size conditions and weaker assumptions.
In the non-convex setting, we prove the almost sure convergence of SGD iterates for step-sizes \( \{ \epsilon_t \}_{t \geq 1} \) satisfying $\sum_{t=1}^{+\infty} \epsilon_t = +\infty$ and $\sum_{t=1}^{+\infty} \epsilon_t^p < +\infty$ for some $p > 2$. Compared with previous studies, our analysis eliminates the global Lipschitz continuity assumption on the loss function and relaxes the boundedness requirements for higher-order moments of stochastic gradients. Building upon the almost sure convergence results, we further establish $L_2$ convergence. These significantly relaxed assumptions make our theoretical results more general, thereby enhancing their applicability in practical scenarios.

\end{abstract}


\section{Introduction}
Stochastic Gradient Descent (SGD), originally presented in the seminal work \cite{1951A}, stands as one of the most widely used optimization algorithms in the fields of machine learning and deep learning, due to its simplicity and remarkable efficiency in handling large-scale datasets (\cite{lecun2002efficient,hinton2012practical,ruder2016overview}). Given these attributes, conducting theoretical research, including convergence analyses, on the SGD is important, as it enables a more effective application. This includes understanding the optimal conditions for its use, how to adjust the parameters to achieve the best results, and determining when it outperforms other optimization methods (\cite{zhou2020towards}). 

The convergence analysis of SGD iterates in the vanishing step-size \footnote{Also known as the step size} setting is primarily based on the Robbins-Monro conditions (\cite{1951A}), which require the step-size sequence \( \{ \epsilon_t \}_{t \geq 1} \) to satisfy the following summability conditions: \(\sum_{t=1}^{+\infty} \epsilon_t = +\infty\) and \(\sum_{t=1}^{+\infty} \epsilon_t^2 < +\infty\). 
However, in practical applications, a broader variety of step sizes is commonly chosen, which often fails to meet the Robbins-Monro conditions, thus lacking corresponding theoretical convergence guarantees.
For instance, when the step-size follows  \(\epsilon_{t} = \mathcal{O}(1/\sqrt{t})\), the theoretical results based on the Robbins-Monro conditions cannot ensure convergence. Nevertheless, despite the strong assumptions this step-size scheme achieved near-optimal sample complexity of \(\mathcal{O}(\ln T/\sqrt{T})\), where \(T\) denotes the total number of iterations. Therefore, the focus of this paper is to explore the convergence of SGD iterates under more relaxed step-size conditions, aiming to bridge the gap between theory and practical applications. 

We summarize the relevant background and our contributions as follows:

\textbf{Related Works.} 
The initial proof of the almost sure convergence of the SGD iterates can be traced back to the works of \cite{ljung1977analysis,ljung1987theory}. They established convergence results under the assumption that the trajectories of the iterates are bounded, i.e.,\(\sup_{t \ge 1} \|\theta_t\| < +\infty~(a.s.)\), where \(\theta_t\) denotes the value of the optimized parameters at step \(t\). However, verifying this assumption in advance is impractical in real-world scenarios. Consequently, in the literature on stochastic approximation, the boundedness of SGD trajectories is often enforced manually, as seen in the works of \cite{benaim2006dynamics,borkar2008stochastic,kushner1997stochastic}, among others. As a result, theoretical findings that rely on this assumption may have limited practical applicability.

The ODE method of stochastic approximation, introduced by \cite{benveniste2012adaptive}, is employed in \cite{mertikopoulos2020almost} to establish the almost sure convergence of the SGD iterates under more relaxed step-size conditions. This study reveals that the assumption of trajectory boundedness is inherently satisfied under several standard conditions, including the global uniform boundedness of the gradients of the loss function (i.e., Lipschitz continuity of the loss function). However, many common optimization scenarios, such as those involving the squared, exponential and logarithmic loss functions, do not satisfy the Lipschitz continuity.
\cite{jin2022revisit} investigates the convergence of SGD iterates under relaxed Robbins-Monro step-size conditions, specifically when \(\sum_{t=1}^{+\infty} \epsilon_t = +\infty\) and \(\sum_{t=1}^{+\infty} \epsilon_t^{2+\delta} < +\infty\), where \(0 \leq \delta < \frac{1}{2}\). However, this work assumes that the set of saddle points of the loss function is empty and that the set of stationary points consists of at most finitely many connected components. These assumptions are difficult to verify in advance, thereby limiting the applicability of this theory in engineering contexts.

\textbf{Our Contributions.}
This paper introduces a novel analytical approach, the stopping-time method based on martingale theory, to establish the asymptotic convergence (i.e., almost sure convergence and $L_2$ convergence) of SGD iterates under weaker assumptions and more relaxed step-size conditions (i.e., \(\sum_{t=1}^{+\infty} \epsilon_t = +\infty\) and \(\sum_{t=1}^{+\infty} \epsilon_t^p < +\infty\) for some \(p > 2\)). The detailed contributions are as follows:

\textbf{1.} We establish the almost sure convergence of the sequence of iterates generated by SGD. Compared with closely related works such as \cite{mertikopoulos2020almost}, our proof removes the assumptions of global Lipschitz continuity of the loss function. Additionally, our approach relaxes the boundedness assumptions on higher-order moments of stochastic gradients. Specifically, we replace the global boundedness assumption on second-order moments with a weaker growth condition and substitute local boundedness conditions for the global boundedness assumption on moments of order greater than two.

\textbf{2.} Furthermore, we prove the \(L_2\) convergence of the SGD iterates under the same assumptions. While almost sure convergence implies \(L_2\) convergence when the loss function is Lipschitz continuous (by the \textbf{Lebesgue Dominated Convergence Theorem}), our analysis eliminates this Lipschitz continuity requirement. Consequently, proving \(L_2\) convergence in such cases remains a challenging problem, as highlighted in Remark \ref{as_vs_L_2}.  

\section{Preliminaries}
\textbf{Problem Formulation:}
Suppose the model parameters are denoted by \(\theta\in \mathbb{R}^{d}\) 
the problem of interest is to minimize the loss function $\min_{\theta\in \mathbb{R}^{d}}f(\theta)$.

\begin{algorithm}[H]
\caption{SGD}
\label{alg:sgd}
\begin{algorithmic}[1]
\Require Initialize $\theta_1$
\For{$t = 1, 2, \ldots, N$}
    \State Compute the stochastic gradient $g_t \gets \nabla f(\theta_t; \xi_t)$
    \State Update the parameter $\theta_{t+1} \gets \theta_t - \epsilon_t \cdot g_t$
\EndFor
\end{algorithmic}
\end{algorithm}
The SGD algorithm is shown in Algorithm\ref{alg:sgd}, where \( \epsilon_t \) denotes the step-size at the \(t\)-th iteration, which can be constant or vary over time. In the $t$-th iteration, the stochastic gradient of the loss function, denoted as \( \nabla f(\theta_t; \xi_t) \), provides an unbiased estimate of the true gradient \( \nabla f(\theta_t) \), based on the sampled random variables \( \{\xi_t\}_{t \geq 1} \), which are mutually independent.
In the subsequent analysis, we denote \( \mathscr{F}_t := \sigma(g_1, \dots, g_{t}) \) as the \( \sigma \)-algebra generated by the stochastic gradients up to the \( t \)-th iteration , with \( \mathscr{F}_0 := \{\Omega, \emptyset\} \) and \( \mathscr{F}_{\infty} := \sigma\left( \bigcup_{t \geq 1} \mathscr{F}_t \right) \).

\subsection{Step-size Conditions}
The relaxed step-size conditions used in our proof are presented in Setting \ref{assump:learning_rate}. Table~\ref{tab:step_size_conditions} compares the relaxed conditions with the Robbins--Monro conditions. Both conditions accommodate step-sizes of the form $\epsilon_t = 1/t^q$ and $\epsilon_t = \log(t)/t^q$ when $q \in (\frac{1}{2},1]$. However, slower-decaying step-sizes with \( q \in \left( \frac{1}{p}, \frac{1}{2} \right] \) satisfy only the relaxed conditions.

\begin{setting}[the relaxed step-size conditions]
\label{assump:learning_rate}
Let \( \{ \epsilon_t \}_{t \geq 1} \) be a sequence of positive monotonic nonincreasing real numbers representing the step sizes used in the optimization algorithm. The sequence \( \{ \epsilon_t \}_{n\ge 1} \) satisfies the following summability conditions:
\[\sum_{t=1}^{+\infty} \epsilon_t = \infty, \quad \text{and} \quad \sum_{t=1}^{+\infty} \epsilon_t^p < \infty\ \ (\text{for some $p>2$}).\]
\end{setting}

\begin{table}[H]
    \centering
    \renewcommand{\arraystretch}{1.2} 
    \begin{tabular}{lccl}
        \toprule
        \textbf{Step-size Schedule} & \textbf{Range of $q$} & \textbf{Robbins--Monro} & \textbf{Relaxed} \\
        \midrule
        \multirow{2}{*}{$\epsilon_t = 1/t^q$} 
            & $q \in (\frac{1}{2},1]$        & \checkmark & \checkmark \\
            & $q \in (\frac{1}{p},\frac{1}{2}]$ & \texttimes & \checkmark \\
        \midrule
        \multirow{2}{*}{$\epsilon_t = \log(t)/t^q$} 
            & $q \in (\frac{1}{2},1]$        & \checkmark & \checkmark \\
            & $q \in (\frac{1}{p},\frac{1}{2}]$ & \texttimes & \checkmark \\
        \bottomrule
    \end{tabular}
    \caption{Comparison of Step-size Conditions.}
    \label{tab:step_size_conditions}
\end{table}

\section{Assumptions and Results}
In this section, we present the basic assumptions required for our proofs and state our main theorems.
\subsection{Assumptions}

\begin{assumpt}[Assumptions on the Loss Function]
\label{assump:loss_function}
Let \( f: \mathbb{R}^d \to \mathbb{R} \) be a $d$-times differentiable function (the loss function). We impose the following conditions:
\begin{enumerate}[label=(\alph*)]
    \item \textbf{Finite Lower Bound:} \label{assump:loss_function_lower_bound} There exists a real number \( f^* \in \mathbb{R} \) such that
    \[f(\theta) \geq f^* \quad \text{for all } \theta \in \mathbb{R}^d.\]
    \item \textbf{Lipschitz Continuous Gradient:} \label{assump:loss_function_lipschitz} The gradient mapping \( \nabla f: \mathbb{R}^d \to \mathbb{R}^d \) is Lipschitz continuous with Lipschitz constant \( L > 0 \); that is, for all \( \theta_1, \theta_2 \in \mathbb{R}^d \),
    \[\| \nabla f(\theta_1) - \nabla f(\theta_2) \| \leq L \| \theta_1 - \theta_2 \|.\]
    \item \textbf{Coercivity:} $\lim_{\|\theta\|\rightarrow+\infty}f(\theta)=+\infty.$
    \item \textbf{Boundedness Near Critical Points:} \label{assump:loss_function_level_set} There exists two constants \( \eta > 0,\ \ D_{\eta} > 0 \) such that the sublevel set containing points with small gradient norm is bounded above in function value; explicitly,
    \[\left\{ \theta \in \mathbb{R}^d \,\middle|\, \| \nabla f(\theta) \| < \eta \right\} \subseteq \left\{ \theta \in \mathbb{R}^d \,\middle|\, f(\theta)-f^* < D_{\eta} \right\}.\]
\end{enumerate}
\end{assumpt}
Assumptions \ref{assump:loss_function} (a) and (b) are classical conditions in SGD analysis, ensuring that the loss function is bounded below and that its gradient does not vary too rapidly. These properties are fundamental for establishing convergence (see, e.g., \cite{bottou2010large, ghadimi2013stochastic}). Since our proof relaxes the Robbins-Monro step-size conditions, Assumptions \ref{assump:loss_function} (c) and (d) are necessary to impose restrictions on the behavior of the loss function at infinity to prevent the algorithm from diverging. 
It is worth noting that the Assumption (d) is weaker than the non-asymptotic flatness assumption in \cite{mertikopoulos2020almost}, i.e., $\liminf_{\|\theta\|\rightarrow+\infty}\|\nabla f(\theta)\|>0$, as it allows for the existence of infinitely distant stationary points with finite function values.
Moreover, compared to the work of \cite{mertikopoulos2020almost}, we do not require the loss function itself to be Lipschitz continuous, that is, we do not impose any boundedness on the gradient.



We make the following assumption regarding the stochastic gradient \( g_t \) .

\begin{assumpt}[Assumptions on the Stochastic Gradient]
\label{assump:stochastic_gradient}
Let \( \{ \theta_t \}_{t\ge 1} \subset \mathbb{R}^d \) be a sequence of iterates generated by SGD, and let \( \{ g_t \}_{t\ge 1} \subset \mathbb{R}^d \) be the corresponding stochastic gradients. We impose the following conditions on \( g_t \):

\begin{enumerate}[label=(\alph*)]
    \item \textbf{Unbiasedness:} \label{assump:stochastic_gradient_unbiased} For all \( t \geq 1 \), \(\mathbb{E}[ g_t \mid \mathscr{F}_{t-1} ] = \nabla f(\theta_t).\)
    \item \textbf{Weak Growth Condition:} \label{assump:stochastic_gradient_weak_growth} There exists a constant \( G > 0 \) such that for all \( t \geq 1 \),
    \[
    \mathbb{E}\left[ \| g_t \|^2 \mid \mathscr{F}_{t-1} \right] \leq G \left( \| \nabla f(\theta_t) \|^2 + 1 \right).
    \]
    \item \textbf{Bounded $p$-th Moment condition in a bounded Region:}
    There exists a fixed constant $M_0>0,$ such that for all \(t\), if \(f(\theta_t)-f^*<D_{\eta},\) where $D_{\eta}$ shown in Assumption \ref{assump:loss_function} Item $(d),$ then the conditional $p$-th moment of the stochastic gradient $g_t$ satisfies
    \[\Expect\left[ \|g_t\|^{p} \mid \mathscr{F}_{t-1} \right] \leq M_{0}^{p}.\]
    
    \item \textbf{Bounded $2p-2$-th Moment condition in a small Region:} There exists two fixed constants $\delta > 0,$ $M_1>0,$ such that for all \(t\), if \(\theta_t \in S_{\delta}\), where  
\[  
S_{\delta} := \left\{ \theta \, \middle| \, \exists \, \theta_1 \in \text{Crit}(f) \text{ with } |f(\theta) - f(\theta_1)| < \delta \right\},  
\]  
then the conditional \(2p-2\)-th moment  of the stochastic gradient \(g_t\) satisfies  
\[  
\Expect\left[ \|g_t\|^{2p-2} \mid \mathscr{F}_{t-1} \right] \leq M_1^{2p-2},
\]
where $\text{Crit}(f) := \{\theta \mid \nabla f(\theta)=0\}$ denotes the set of stationary points of the function $f$.
\end{enumerate}
\end{assumpt}
Assumptions \ref{assump:stochastic_gradient} (a) and (b) are fairly standard in the analysis of SGD (see, e.g.,\cite{bottou2010large, ghadimi2013stochastic, bottou2018optimization, nguyen2018sgd}). Assumption \ref{assump:stochastic_gradient} (a), the unbiasedness of the stochastic gradient, ensures that its expectation equals the true gradient. Although a large variance does not alter the expected direction of the update, it may lead to oscillations and reduce optimization efficiency. Therefore, in addition to the unbiasedness assumption, the analysis of SGD typically requires additional variance control, such as the bounded second-moment assumption. Assumption~\ref{assump:stochastic_gradient}(b), the weak growth condition, addresses this by relating the second moment of the stochastic gradient to the norm of the true gradient, thereby preventing excessive variance that could hinder convergence.
Assumptions~\ref{assump:stochastic_gradient}(c) and (d) impose local boundedness conditions on the higher-order moments of the stochastic gradient \(g_t\).

Table~\ref{tab:assumption_comparison} summarizes the assumptions required by our work compared to those in the closely related work of \cite{mertikopoulos2020almost}. While \cite{mertikopoulos2020almost} requires the loss function to be Lipschitz continuous, our analysis removes this assumption entirely. Both works assume Lipschitz smoothness, coercivity, and unbiased stochastic gradients. However, we substantially relax several key assumptions:
\begin{itemize}
    \item Non-asymptotic flatness: \cite{mertikopoulos2020almost} assumes a global non-asymptotic flatness condition, while our work weakens this to a boundedness condition near critical points.
    \item Boundedness of second-order moments: Instead of assuming global boundedness of the second-order moments of stochastic gradients, we impose a weaker condition—namely, a weak growth condition.
    \item Boundedness of higher-order moments: The global boundedness assumption on higher-order (order greater than 2) moments is replaced in our work by local boundedness conditions, further relaxing the constraints required on stochastic gradients.
\end{itemize}

These relaxed conditions make our theoretical analysis more applicable to practical scenarios, especially in cases where the global Lipschitz continuity and stringent boundedness conditions of stochastic gradients are unrealistic.

\begin{table}[ht]
    \footnotesize
    \centering
    \renewcommand{\arraystretch}{1.2} 
        \begin{tabular}{lcc}
            \toprule
            \textbf{Assumptions} & \textbf{Mertikopoulos et al. (2020)} & \textbf{Our Works} \\
            \midrule
            Lipschitz Continuity  & \checkmark & \texttimes \\
            Lipschitz Smoothness      & \checkmark & \checkmark \\
            Coercivity            & \checkmark & \checkmark \\
            Non-asymptotic flatness          & \checkmark & Boundedness Near Critical Point  \\
            Unbiasedness          & \checkmark & \checkmark \\
            Boundedness of Second Moment         & Globally Bounded & Weak Growth Condition  \\
            Boundedness of Higher Moment         & Globally Bounded & Locally Bounded  \\
            \bottomrule
        \end{tabular}
    \caption{Comparison of Assumptions: Mertikopoulos et al. (2020) vs. Our Works . \newline \textit{Note: Although not explicitly shown in the table, our function value convergence proof only requires local boundedness of the \(p\)-th moment, while related works require global boundedness of the \(2p-2\)-th moment. 
   Similarly, for gradient convergence, we only assume the local boundedness of the \(2p{-}2\)-th moment within a region whose image under the mapping \(f\) has Lebesgue measure controlled by a constant \(\delta > 0\), which can be made arbitrarily small.}}
    \label{tab:assumption_comparison}
\end{table}

\subsection{Main Theorems}
Here are our main results, establishing the asymptotic convergence of SGD under the assumptions outlined above.
Under Assumptions \ref{assump:loss_function} and \ref{assump:stochastic_gradient}(a)-(c), we first establish the convergence of the function values to a critical value. With the additional assumption 3.2(d), we further prove the convergence of the gradient norm.

\subsection{Convergence of The Loss Function Value}

In this subsection, we explore the convergence properties of the loss function value during the training process.

\begin{thm}\label{thm:as_convergence:-1}
Let \( \{ \theta_t \}_{t \ge 1} \subset \mathbb{R}^d \) be the sequence generated by Algorithm \ref{alg:sgd} with the initial point \( \theta_1 \). Suppose the step sizes \( \{\epsilon_t\}_{t \ge 1} \) satisfy the conditions in Setting \ref{assump:learning_rate}, and Assumptions \ref{assump:loss_function} and \ref{assump:stochastic_gradient} (a)-(c) hold. Then, the sequence \( \{f(\theta_{t})\}_{t \ge 1} \) converges almost surely to some critical value \( f(\theta^*) \), where \( \theta^* \in \text{Crit} f \).
\end{thm} 

\noindent This result provides a guarantee that the sequence of loss function values \( \{f(\theta_t)\}_{t \ge 1} \) will converge almost surely under the given assumptions. Notably, in deriving this result, we only require a local boundedness condition on the $p$-order moments , which constitutes a significant improvement over the result of \cite{mertikopoulos2020almost}, where a global boundedness condition with $2p - 2$  is assumed.

\begin{proof}
To prove the sequence \(\{f(\theta_t)\}_{t \geq 1}\) almost surely converge to a critical value, it suffices to prove following two statements:
\begin{enumerate}[label=(\alph*)]
\item $\liminf_{t\rightarrow+\infty }\|\nabla f(\theta_{t})\|=0\ \ \text{a.s.}$
\item $\liminf_{t\rightarrow+\infty }f(\theta_{t})-f^*< C_{\eta\
}\ \ \text{a.s.}$, and for any \( x > 0 \) and \( \delta > 0 \), consider the open interval \( I_{x,\delta} = (x, x + \delta) \) with \( x \in [0, D_{\eta}) \). The sequence \( \{ f(\theta_t) \}_{t \geq 1} \) almost surely up-crosses~\footnote{We formalize the definition of an up-crossing of an interval for a real sequence \( \{x_t\}_{t \geq 1} \) in this paper.
\begin{definition}[Up-crossing of an interval]
Let \( \{x_t\}_{t\ge 1} \) be a real sequence, and let \( e \) and \( o \) be the left and right endpoints of an interval, respectively (without any restrictions on whether the interval is open or closed). We say that the sequence \( \{x_t\}_{t \geq 1} \) completes a single up-crossing of the interval during the time period \( t_1 \leq t \leq t_2 \) if and only if the following conditions are satisfied:
\begin{enumerate}
    \item \( x_{t_1} < e \),
    \item \( x_{t_2} \geq o \),
    \item If the set \( \{ t \mid t_1 < t < t_2 \} \) is nonempty, then for all \( t_1 < t < t_2 \), it holds that
    \[
    e \leq x_t < o.
    \]
\end{enumerate}
\end{definition}} this interval only finitely many times.

\end{enumerate}
It can be seen that Statement (a) ensures the existence of a subsequence of iterates that converges to a critical point. This further implies that the sequence of loss function values \( \{f(\theta_t)\}_{t \ge 1} \) has a subsequence that converges to a critical value. On the other hand, Statement (b) guarantees that the sequence of loss function values \( \{f(\theta_t)\}_{t \ge 1} \) converges, which means that the critical value to which this subsequence converges is the limit of the entire sequence.

We will now proceed to prove these statements.

\noindent{\bf Phase I:} (Proving Statement (a))

This statement actually indicates that there exists a subsequence such that the gradient norm converges to $0.$ We present this result as a lemma below. The proof of this lemma is relatively straightforward and is provided in Appendix \ref{dsfadsfdsafdasdfasdfasdfaadas}.
\begin{lem}\label{thm:as_convergence:-1.0.0}
Let \( \{ \theta_t \}_{t \geq 1} \subset \mathbb{R}^d \) be the sequence generated by Algorithm \ref{alg:sgd} starting from the initial point \( \theta_1 \). Suppose the step size sequence \( \{\epsilon_t\}_{t \geq 1} \) satisfies the conditions specified in Setting \ref{assump:learning_rate}. Additionally, assume that Assumptions \ref{assump:loss_function} (a)-(b) and \ref{assump:stochastic_gradient} (a)-(b) are satisfied. Then, the sequence \( \{ \|\nabla f(\theta_{t})\| \}_{t \geq 1} \) contains a subsequence that converges to \( 0 \) almost surely. Specifically, we have  
\[
\liminf_{t \to +\infty} \|\nabla f(\theta_{t})\| = 0 \quad \text{a.s.}
\]
\end{lem}

\noindent{\bf Phase II:} (Proving Statement (b))

Based on Lemma \ref{thm:as_convergence:-1.0.0} and the conditions of the function we are considering, i.e., Assumption \ref{assump:loss_function}~Item(d), we can immediately obtain the first part of statement (b), namely,
\[\liminf_{t\rightarrow+\infty }f(\theta_{t})-f^*< C_{\eta\
}\ \ \text{a.s.}\]
Next, we focus on proving the second half of statement (b), namely, the upper crossing interval part. 

Since \( f \) is \( d \)-times differentiable, by Sard's theorem \cite{sard1942measure, bates1993toward}, we know that the Lebesgue measure of \( f(\text{Crit}(f)) \) is zero, i.e., \( m(f(\text{Crit}(f))) = 0 \), where \( m(\cdot) \) denotes Lebesgue measure and \( f(\text{Crit}(f)) \) represents the image of \( \text{Crit}(f) \) under \( f .\) Furthermore, by the corecivity assumption (Assumption \ref{assump:loss_function} (c)), \( f(\text{Crit}(f)) \) is compact. Therefore, \( f(\text{Crit}(f)) \) is nowhere dense. For any open interval \( I_{x,\delta} \), we can find a smaller closed interval \([h_1,h_2]:= H_{x} \subset I_{x,\delta} \) such that \( H_{x} \cap f(\text{Crit}(f)) = \emptyset \).

Since \( f \) is a continuous map, the preimage \( f^{-1}(H_x) \) of \( H_x \) is also a closed set. Additionally, due to the corecivity property, it is easy to see that \( f^{-1}(H_x) \) is bounded, which makes \( f^{-1}(H_x) \) is compact. Next, we consider the continuous map \( \|\nabla f(\cdot)\|: \mathbb{R}^d \to \mathbb{R} \), and its image on the set \( f^{-1}(H_x) \), which is \( \|\nabla f(f^{-1}(H_x))\| \). Since \( f^{-1}(H_x) \) is compact, we know that the set \( \|\nabla f(f^{-1}(H_x))\| \) attains a minimum value. Moreover, since \( f^{-1}(H_x) \) does not contain critical points, we can easily conclude that this minimum value is positive. Therefore, we define this minimum value as \( \delta_0 := \min\left\{\|\nabla f(f^{-1}(H_x))\|\right\} > 0 \).

For an interval \( X \), let \( U_{X,T} \) denote the total number of times the sequence \( \{f(\theta_t)\}_{1 \leq t \leq T} \) up-crosses \( X \). Since \( H_x \subset I_{x,\delta} \), it follows directly that every up-crossing of \(I_{x,\delta} \) is also a up-crossing of \( H_x  \). Therefore, we have the inequality:
\[
U_{I_{x,\delta},T} \leq U_{H_x, T}.
\]
Thus, it suffices to establish that  
\[
\limsup_{T\rightarrow+\infty}U_{H_x,T}<+\infty \quad \text{a.s.}
\]
Since \( U_{H_x,T} \) is monotonically increasing with respect to the index \( T \), by the \emph{Lebesgue's Monotone Convergence Theorem}, it suffices to prove that  
\[
\limsup_{T \to +\infty} \mathbb{E} \left[ U_{H_x,T} \right] < +\infty.
\]
To this end, we construct the following sequence of stopping times \( \{\mu_n(h_1,h_2)\}_{n \geq 1} \):
\begin{align}\label{qwerewqwq}
&\mu_{1}(h_1,h_2):=\min\{t\ge 1:f(\theta_{t})-f^*\ge h_1\},\notag\\& \mu_{2}(h_1,h_2):=\min\{t\ge \mu_{1}(h_1,h_2):f(\theta_{t})-f^*\ge h_2,\ \text{or}\ f(\theta_{t})-f^*< h_1\},\notag\\&\mu_{3}(h_1,h_2):=\min\{t\ge \mu_{2}(h_1,h_2): f(\theta_{t})-f^*< h_1\}\notag\\&\mu_{3k-2}(h_1,h_2):=\min\{t\ge \mu_{2k-2}(h_1,h_2):f(\theta_{t})-f^*\ge h_1\},\notag\\&\mu_{3k-1}(h_1,h_2):=\min\{t\ge \mu_{2k-1}(h_1,h_2):f(\theta_{t})-f^*\ge h_2,\ \text{or}\ f(\theta_{t})-f^*< h_1\},\notag\\&\mu_{3k}(h_1,h_2):=\min\{t\ge \mu_{2k-1}(h_1,h_2):f(\theta_{t})-f^*< h_1\}.
\end{align}
Then for any fixed deterministic time \( T \), we define \(\mu_{n,T}(h_1,h_2) := \mu_n(h_1,h_2) \wedge T.\) To keep the notation concise, in the following proof up to Eq.~\ref{dasd123213}, we choose to omit \((h_1, h_2)\) and simplify \( \mu_n(h_1, h_2) ,\ \mu_{n,T}(h_1,h_2)\) to \(\mu_{n},\ \mu_{n,T}.\)

For \( U_{H_x,T} \), we can readily establish the following inequality:
\[
U_{H_x,T}\leq 1+{\sum_{k=2}^{+\infty} \I_{f(\theta_{\mu_{3k-1,T}})-f^*\ge h_2}},
\] which means
\begin{align}\label{sdfsadfwef}\Expect\left[U_{H_x,T}\right]\leq 1+\underbrace{\sum_{k=2}^{+\infty} \Expect\left[\I_{f(\theta_{\mu_{3k-1,T}})-f^*\ge h_2}\right]}_{\Sigma_{T,\mu}}.\end{align}
Below, we explain why the summation in the above equation starts from \( k = 2 \) and why the additional \( +1 \) term is included.  

In the first cycle, i.e., within \( \mu_1, \mu_2, \mu_3 \), we cannot determine the relationship between \( f(\theta_1) - f^* \) and \( h_1 \). As a result, we cannot assert whether an up-crossing of the interval has occurred solely based on whether \( f(\theta_{\mu_2}) - f^* \geq h_2 \). Specifically, if \( f(\theta_1) - f^* \geq h_1 \), then even if \( f(\theta_{\mu_2}) - f^* \geq h_2 \), the first cycle still does not complete an up-crossing.  

However, from the second cycle onward, the definition of \( \mu_{3k} \) for \( k \geq 1 \) ensures that \( f(\theta_{\mu_3}) - f^* \) is below \( h_1 \). Therefore, starting from the second cycle, whether an up-crossing occurs depends entirely on whether \( f(\theta_{\mu_{3k-1}}) - f^* \) is at least \( h_2 \). This explains why the summation starts from \( k = 2 \).  

As for the additional \( +1 \) term, it accounts for the fact that at most one up-crossing can occur within the first cycle in the interval \( [h_1, h_2] \). Thus, adding 1 ensures the validity of the inequality.

Next, we focus on bounding \( \Sigma_{T,\mu}.\) We have:
\begin{align}\label{sgd__0}
\Sigma_{T,\mu}&\mathop{\le}^{(a)} \frac{1}{h_2-h_1}\sum_{k=2}^{+\infty}\Expect\left[\I_{f(\theta_{\mu_{3k-1,T}})-f^*\ge h_2}\left(f(\theta_{\mu_{3k-1,T}})-f(\theta_{\mu_{3k-2,T}-1})\right)\right]\notag\\&= \frac{1}{h_2-h_1}\sum_{k=2}^{+\infty}\Expect\left[\I_{f(\theta_{\mu_{3k-1,T}})-f^*\ge h_2}\sum_{t=\mu_{3k-2,T}-1}^{\mu_{3k-1,T}-1}\left(f(\theta_{t+1})-f(\theta_{t})\right)\right]\notag\\&=\frac{1}{h_2 -h_1}\underbrace{\sum_{k=2}^{+\infty}\Expect\left[\I_{f(\theta_{\mu_{3k-1,T}})-f^*\ge h_2}\left(f(\theta_{\mu_{3k-2,T}})-f(\theta_{\mu_{3k-2,T}-1})\right)\right]}_{\Sigma_{T,\mu,1}}\notag\\&+\frac{1}{h_2-h_1}\underbrace{\sum_{k=2}^{+\infty}\Expect\left[\I_{f(\theta_{\mu_{3k-1,T}})-f^*\ge h_2}\sum_{t=\mu_{3k-2,T}}^{\mu_{3k-1,T}-1}\left(f(\theta_{t+1})-f(\theta_{t})\right)\right]}_{\Sigma_{T,\mu,2}}.
\end{align}
Next, we first bound \( \Sigma_{T,\mu,1}.\) To proceed, we introduce the following lemma (the proof of which is provided in Appendix \ref{lem+1123}):
\begin{lem}
\label{lem_1123}
Let \( \{ \theta_t \}_{t \geq 1} \subset \mathbb{R}^d \) be the sequence generated by Algorithm \ref{alg:sgd} with the initial point \( \theta_1 \), and assume that the step size sequence \(\{\epsilon_t\}_{t \geq 1}\) satisfies Setting \ref{assump:learning_rate}. Under Assumption \ref{assump:loss_function} (a), (b) and Assumption \ref{assump:stochastic_gradient} (a)-(c), the following inequality holds for any \( \nu > 0 \):
  \[
  \sum_{t=1}^{+\infty} \mathbb{E} \left[\overline{\Delta}_{t,\nu}\right] \leq C_{\nu} < +\infty.
  \]
where
\[
\overline{\Delta}_{t,\nu} := \left( \mathbb{I}_{f(\theta_t) - f^* < D_{\eta}}  \left| f(\theta_{t+1}) - f(\theta_t) \right| - \nu \right)_{+},
\]
and \( C_{\nu} \) is a constant depending only on \( \nu \).
\end{lem}
Therefore, we can apply the aforementioned lemma to estimate the term \( \Sigma_{T,\mu,1} \). Specifically, by setting  
\(\nu = ({h_2 - h_1})\big/{8}\) in Lemma \ref{lem_1123}, we obtain:

\begin{align} \label{sgd_043241}
\Sigma_{T,\mu,1}
&\overset{(\text{i})}{\le}
\sum_{k=2}^{+\infty} \mathbb{E} \Big[
    \mathbb{I}_{f(\theta_{\mu_{3k-1,T}}) - f^* \ge h_2} \,
    \mathbb{I}_{f(\theta_{\mu_{3k-2,T}-1}) - f^* < D_{\eta}} \,
    \left| f(\theta_{\mu_{3k-2,T}}) - f(\theta_{\mu_{3k-2,T}-1}) \right|
\Big] \notag \\
&= \sum_{k=2}^{+\infty} \mathbb{E} \Big[
    \mathbb{I}_{f(\theta_{\mu_{3k-1,T}}) - f^* \ge h_2} \,
    \mathbb{I}_{f(\theta_{\mu_{3k-2,T}-1}) - f^* < D_{\eta}} \,
    \Big( 
        \left| f(\theta_{\mu_{3k-2,T}}) - f(\theta_{\mu_{3k-2,T}-1}) \right|
        - \tfrac{h_2 - h_1}{8}
    \Big)
\Big] \notag \\
&\quad + \frac{h_2 - h_1}{8}
\sum_{k=2}^{+\infty} \mathbb{E} \Big[
    \mathbb{I}_{f(\theta_{\mu_{3k-1,T}}) - f^* \ge h_2} \,
    \mathbb{I}_{f(\theta_{\mu_{3k-2,T}-1}) - f^* < D_{\eta}}
\Big] \notag \\
&\le \sum_{k=2}^{+\infty} \mathbb{E} \Big[
    \mathbb{I}_{f(\theta_{\mu_{3k-1,T}}) - f^* \ge h_2} \,
    \mathbb{I}_{f(\theta_{\mu_{3k-2,T}-1}) - f^* < D_{\eta}} \,
    \Big( 
        \left| f(\theta_{\mu_{3k-2,T}}) - f(\theta_{\mu_{3k-2,T}-1}) \right|
        - \tfrac{h_2 - h_1}{8}
    \Big)_{+}
\Big] \notag \\
&\quad + \frac{h_2 - h_1}{8}
\sum_{k=2}^{+\infty} \mathbb{E} \Big[
    \mathbb{I}_{f(\theta_{\mu_{3k-1,T}}) - f^* \ge h_2} \,
    \mathbb{I}_{f(\theta_{\mu_{3k-2,T}-1}) - f^* < D_{\eta}}
\Big] \notag \\
&\overset{(\text{ii})}{\le}
C_{(h_2 - h_1)/8} + \frac{h_2 - h_1}{8} \Sigma_{T,\mu}.
\end{align}

In the above derivation, step \((\text{i})\) first requires noting that \( \I_{f(\theta_{\mu_{3k-2,T}-1}) - f^* < D_{\eta}}= \mathbf{1}\ \ (\forall\ k\ge 2) \), and that \( f(\theta_{\mu_{3k-2,T}})-f(\theta_{\mu_{3k-2,T}-1}) \le \left| f(\theta_{\mu_{3k-2,T}})-f(\theta_{\mu_{3k-2,T}-1}) \right|.\) For step \((\text{ii})\), we apply Lemma \ref{lem_1123} with \( \nu = (h_2 - h_1)/8 \) to the first term on the left-hand side of the inequality involving \( C_{(h_2-h_1)/8} \). Additionally, for the term \( \frac{h_2 - h_1}{8} \Sigma_{T,\mu} \), it is crucial to note that \(\mathbb{I}_{f(\theta_{\mu_{3k-2,T}-1}) - f^* < D_{\eta}} = \mathbf{1}, \quad \forall\ k \geq 2,\) that is,
\begin{align*}
&\quad\frac{h_2-h_1}{8}\sum_{k=2}^{+\infty}\Expect\left[\I_{f(\theta_{\mu_{3k-1,T}})-f^*\ge h_2}\I_{f(\theta_{\mu_{3k-2,T}-1}) - f^* < D_{\eta}}\right]\\
&=\frac{h_2-h_1}{8}\sum_{k=2}^{+\infty}\Expect\left[\I_{f(\theta_{\mu_{3k-1,T}})-f^*\ge h_2}\right]\\&=\frac{h_2-h_1}{8}\Sigma_{T,\mu}.
\end{align*}
Next, we proceed to bound \( \Sigma_{T,\mu,2} \) in Eq. \ref{sgd__0}. We have:
\begin{align} \label{sgd__3'}
\Sigma_{T,\mu,2}
&\overset{\text{Lemma~\ref{descent_lemma}}}{=}
\sum_{k=2}^{+\infty} \mathbb{E} \Bigg[
    \mathbb{I}_{f(\theta_{\mu_{3k-1,T}}) - f^* \ge h_2}
    \sum_{t=\mu_{3k-2,T}}^{\mu_{3k-1,T}-1}
    \left(
        -\epsilon_t \|\nabla f(\theta_t)\|^2
        + M_t + \frac{L \epsilon_t^2}{2} \|g_t\|^2
    \right)
\Bigg] \notag \\
&\le 
\sum_{k=2}^{+\infty} \mathbb{E} \Bigg[
    \mathbb{I}_{f(\theta_{\mu_{3k-1,T}}) - f^* \ge h_2}
    \sum_{t=\mu_{3k-2,T}}^{\mu_{3k-1,T}-1}
    \left(
        M_t + \frac{L \epsilon_t^2}{2} \|g_t\|^2
    \right)
\Bigg] \notag \\
&\le 
\sum_{k=2}^{+\infty} \mathbb{E} \Bigg[
    \mathbb{I}_{f(\theta_{\mu_{3k-1,T}}) - f^* \ge h_2}
    \sum_{t=\mu_{3k-2,T}}^{\mu_{3k-1,T}-1} M_t
\Bigg]
+ \frac{L}{2} \sum_{k=2}^{+\infty} \mathbb{E} \left[
    \sum_{t=\mu_{3k-2,T}}^{\mu_{3k-1,T}-1} \epsilon_t^2 \|g_t\|^2
\right] \notag \\
&\overset{\text{\emph{AM--GM Inequality}}}{\le}
\frac{h_2 - h_1}{8}
\sum_{k=2}^{+\infty} \mathbb{E} \left[
    \mathbb{I}_{f(\theta_{\mu_{3k-1,T}}) - f^* \ge h_2}
\right] \notag \\
&\quad + \frac{2}{h_2 - h_1}
\sum_{k=2}^{+\infty} \mathbb{E} \left[
    \left(
        \sum_{t=\mu_{3k-2,T}}^{\mu_{3k-1,T}-1} M_t
    \right)^2
\right]
+ \frac{L}{2}
\sum_{k=2}^{+\infty} \mathbb{E} \left[
    \sum_{t=\mu_{3k-2,T}}^{\mu_{3k-1,T}-1} \epsilon_t^2 \|g_t\|^2
\right] \notag \\
&\le
\frac{h_2 - h_1}{8}
\sum_{k=2}^{+\infty} \mathbb{E} \left[
    \mathbb{I}_{f(\theta_{\mu_{3k-1,T}}) - f^* \ge h_2}
\right] \notag \\
&\quad + \frac{2}{h_2 - h_1}
\underbrace{
    \sum_{k=2}^{+\infty} \mathbb{E} \left[
        \left(
            \sum_{t=\mu_{3k-2,T}}^{\mu_{3k-1,T}-1} M_t
        \right)^2
    \right]
}_{\Sigma_{T,\mu,2,1}}
+ \frac{L}{2}
\underbrace{
    \sum_{k=2}^{+\infty} \mathbb{E} \left[
        \sum_{t=\mu_{3k-2,T}}^{\mu_{3k-1,T}-1} \epsilon_t^2 \|g_t\|^2
    \right]
}_{\Sigma_{T,\mu,2,2}}.
\end{align}
We now proceed to handle \( \Sigma_{T,\mu,2,1} \) and \( \Sigma_{T,\mu,2,2} \). 

For \( \Sigma_{T,\mu,2,1} \), we have:
\begin{align}\label{sgd__2'}
 \Sigma_{T,\mu,2,1}&\mathop{=}^{\text{\emph{Doob's Stopped Theorem}}} \Expect\left[\sum_{t=\mu_{3k-2,T}}^{\mu_{3k-1,T}-1}\epsilon_{t}^{2}\Expect\left[\|g_{t}\|^{2}|\mathscr{F}_{t-1}\right]\right]\notag\\&\mathop{\le}^{\text{Assumption \ref{assump:stochastic_gradient} Item $(b)$}}  G\Expect\left[\sum_{t=\mu_{3k-2,T}}^{\mu_{3k-1,T}-1}\epsilon_{t}^{2}\left(\|\nabla f(\theta_{t})\|^{2}+1\right)\right]\notag\\&\mathop{\le}^{(\text{i})}\frac{2G}{\delta^{2}_{0}}\Expect\left[\sum_{t=\mu_{3k-2,T}}^{\mu_{3k-1,T}-1}\epsilon_{t}^{2}\|\nabla f(\theta_{t})\|^{2}\right].
\end{align}
The detailed derivation for step \((\text{i})\) is as follows:
\begin{align*}
&\quad\Expect\left[\sum_{t=\mu_{3k-2,T}}^{\mu_{3k-1,T}-1}\epsilon_{t}^{2}\left(\|\nabla f(\theta_{t})\|^{2}+1\right)\right]=\Expect\left[\I_{\mu_{3k-2,T}=\mu_{3k-1,T}}\sum_{t=\mu_{3k-2,T}}^{\mu_{3k-1,T}-1}\epsilon_{t}^{2}\left(\|\nabla f(\theta_{t})\|^{2}+1\right)\right]\\&+\Expect\left[\I_{\mu_{3k-2,T}<\mu_{3k-1,T}}\sum_{t=\mu_{3k-2,T}}^{\mu_{3k-1,T}-1}\epsilon_{t}^{2}\left(\|\nabla f(\theta_{t})\|^{2}+1\right)\right]\\&=0+\Expect\left[\I_{\mu_{3k-2,T}<\mu_{3k-1,T}}\sum_{t=\mu_{3k-2,T}}^{\mu_{3k-1,T}-1}\epsilon_{t}^{2}\left(\|\nabla f(\theta_{t})\|^{2}+1\right)\right]\\&\mathop{\le}^{(\#)} \left(1+\frac{1}{\delta^{2}_{0}}\right)\Expect\left[\I_{\mu_{3k-2,T}<\mu_{3k-1,T}}\sum_{t=\mu_{3k-2,T}}^{\mu_{3k-1,T}-1}\epsilon_{t}^{2}\|\nabla f(\theta_{t})\|^{2}\right]\\&\le \left(1+\frac{1}{\delta^{2}_{0}}\right)\Expect\left[\sum_{t=\mu_{3k-2,T}}^{\mu_{3k-1,T}-1}\epsilon_{t}^{2}\|\nabla f(\theta_{t})\|^{2}\right].
\end{align*}
In the above derivation, step \((\#)\) requires noting that when \( \mu_{3k-2,T}<\mu_{3k-1,T}\), we simultaneously have \(\mu_{3k-2} < T \) and \( \mu_{3k-2} < \mu_{3k-1} \). This implies that for any \( t \in [\mu_{3k-2,T},\mu_{3k-1,T}-1] \), it holds that \( \|\nabla f(\theta_{t})\| \ge \delta_{0}.\) 

For \( \Sigma_{T,\mu,2,2} \), we have:
\begin{align}\label{sgd__3''}
 \Sigma_{T,\mu,2,2}&\mathop{=}^{\text{\emph{Doob's Stopped Theorem}}} \Expect\left[\sum_{t=\mu_{3k-2,T}}^{\mu_{3k-1,T}-1}\Expect\left[M^{2}_{t}|\mathscr{F}_{t-1}\right]\right]\notag\\&= \Expect\left[\I_{\mu_{3k-2,T}=\mu_{3k-1,T}}\sum_{t=\mu_{3k-2,T}}^{\mu_{3k-1,T}-1}\Expect\left[M^{2}_{t}|\mathscr{F}_{t-1}\right]\right]\notag\\&\quad+ \Expect\left[\I_{\mu_{3k-2,T}<\mu_{3k-1,T}}\sum_{t=\mu_{3k-2,T}}^{\mu_{3k-1,T}-1}\Expect\left[M^{2}_{t}|\mathscr{F}_{t-1}\right]\right]\notag\\&=0+\Expect\left[\I_{\mu_{3k-2,T}<\mu_{3k-1,T}}\sum_{t=\mu_{3k-2,T}}^{\mu_{3k-1,T}-1}\Expect\left[M^{2}_{t}|\mathscr{F}_{t-1}\right]\right]\notag\\&=\Expect\left[\I_{\mu_{3k-2,T}<\mu_{3k-1,T}}\sum_{t=\mu_{3k-2,T}}^{\mu_{3k-1,T}-1}\epsilon^{2}_{t}\|\nabla f(\theta_{t})\|^{2}\Expect\left[\|\nabla f(\theta_{t})-g_{t}\|^{2}|\mathscr{F}_{t-1}\right]\right]\notag\\&\mathop{\le}^{(\text{i})}2Lh_{b}\Expect\left[\I_{\mu_{3k-2,T}<\mu_{3k-1,T}}\sum_{t=\mu_{3k-2,T}}^{\mu_{3k-1,T}-1}\epsilon_{t}^{2}\Expect\left[\|g_{t}\|^{2}|\mathscr{F}_{t-1}\right]\right]\notag\\&\mathop{\le}^{\text{Assumption \ref{assump:stochastic_gradient} Item $(b)$}}2Lh_{b}G\Expect\left[\I_{\mu_{3k-2,T}<\mu_{3k-1,T}}\sum_{t=\mu_{3k-2,T}}^{\mu_{3k-1,T}-1}\epsilon_{t}^{2}\left(\|\nabla f(\theta_{t})\|^{2}+1\right)\right]\notag\\&\mathop{\le}^{(\text{ii})}2Lh_{b}G\left(1+\frac{1}{\delta_{0}^{2}}\right)\Expect\left[\I_{\mu_{3k-2,T}<\mu_{3k-1,T}}\sum_{t=\mu_{3k-2,T}}^{\mu_{3k-1,T}-1}\epsilon_{t}^{2}\|\nabla f(\theta_{t})\|^{2}\right]\notag\\&\le 2Lh_{b}G\left(1+\frac{1}{\delta_{0}^{2}}\right)\Expect\left[\sum_{t=\mu_{3k-2,T}}^{\mu_{3k-1,T}-1}\epsilon_{t}^{2}\|\nabla f(\theta_{t})\|^{2}\right].
\end{align}
The detailed derivation of step \((\text{i})\) is as follows:
\begin{align*}
 &\quad\I_{\mu_{3k-2,T}<\mu_{3k-1,T}}\sum_{t=\mu_{3k-2,T}}^{\mu_{3k-1,T}-1}\epsilon^{2}_{t}\|\nabla f(\theta_{t})\|^{2}\Expect\left[\|\nabla f(\theta_{t})-g_{t}\|^{2}|\mathscr{F}_{t-1}\right]\\&\le  \I_{\mu_{3k-2,T}<\mu_{3k-1,T}}2L\sum_{t=\mu_{3k-2,T}}^{\mu_{3k-1,T}-1}\epsilon_{t}^{2}(f(\theta_{t})-f^*)\Expect\left[\|\nabla f(\theta_{t})-g_{t}\|^{2}|\mathscr{F}_{t-1}\right] ,
\end{align*}
We need to note that  when \( \mu_{3k-2,T}<\mu_{3k-1,T}\), we simultaneously have \(\mu_{3k-2} < T \) and \( \mu_{3k-2} < \mu_{3k-1} \). This implies that for any \( t \in [\mu_{3k-2,T},\mu_{3k-1,T}-1] \), it holds that \( f(\theta_{t})-f^*\le h_{b}.\) Then we have:
\begin{align*}
 &\quad  \I_{\mu_{3k-2,T}<\mu_{3k-1,T}}2L\sum_{t=\mu_{3k-2,T}}^{\mu_{3k-1,T}-1}\epsilon_{t}^{2}(f(\theta_{t})-f^*)\Expect\left[\|\nabla f(\theta_{t})-g_{t}\|^{2}|\mathscr{F}_{t-1}\right]\notag\\&\le \I_{\mu_{3k-2,T}<\mu_{3k-1,T}}2Lh_{b}\sum_{t=\mu_{3k-2,T}}^{\mu_{3k-1,T}-1}\epsilon_{t}^{2}\Expect\left[\|\nabla f(\theta_{t})-g_{t}\|^{2}|\mathscr{F}_{t-1}\right]\notag\\&\le \I_{\mu_{3k-2,T}<\mu_{3k-1,T}}2Lh_{b}\sum_{t=\mu_{3k-2,T}}^{\mu_{3k-1,T}-1}\epsilon_{t}^{2}\Expect\left[\|g_{t}\|^{2}|\mathscr{F}_{t-1}\right] .
\end{align*}
For step \((\text{ii})\), we similarly need to note that when \( \mu_{3k-2,T}<\mu_{3k-1,T}\), we simultaneously have \(\mu_{3k-2} < T \) and \( \mu_{3k-2} < \mu_{3k-1} \). This implies that for any \( t \in [\mu_{3k-2,T},\mu_{3k-1,T}-1] \), it holds that \( \|\nabla f(\theta_{t})\| \ge \delta_{0}.\) 

Substituting the estimates of \( \Sigma_{T,\mu,2,1} \) and \( \Sigma_{T,\mu,2,2} \) from Eq. \ref{sgd__2'} and Eq. \ref{sgd__3''} back into Eq. \ref{sgd__3'}, we obtain:
\begin{align}\label{dsadsakl}
\Sigma_{T,\mu,2}&\le\frac{h_2 -h_1}{8}\sum_{k=2}^{+\infty}\Expect\left[\I_{f(\theta_{\mu_{3k-1,T}})-f^*\ge h_2}\right]\notag\\&\quad+\underbrace{\left(\frac{4G}{(h_{b}-h_{a})\delta_{0}^{2}}+{L^{2}h_{b}G}\left(1+\frac{1}{\delta_{0}^{2}}\right)\right)}_{C_{0}(h_1 ,h_2)}\Expect\left[\sum_{t=\mu_{3k-2,T}}^{\mu_{3k-1,T}-1}\epsilon^{2}_{t}\|\nabla f(\theta_{t})\|^{2}\right]\notag\\&{=}\frac{h_2 -h_1}{8}\Sigma_{T,\mu}+{C_{0}(h_1 ,h_2)}\Expect\left[\sum_{t=\mu_{3k-2,T}}^{\mu_{3k-1,T}-1}\epsilon^{2}_{t}\|\nabla f(\theta_{t})\|^{2}\right].
\end{align}
Next, we substitute the estimates for \( \Sigma_{T,\mu,1} \) from Eq. \ref{sgd_043241} and the estimates for \( \Sigma_{T,\mu,2} \) from Eq. \ref{dsadsakl} into Eq. \ref{sgd__0}, yielding:
\begin{align*}
\Sigma_{T,\mu}&{\le} \frac{1}{h_2 -h_1}C_{(h_2-h_1)/8}+\frac{1}{4}\Sigma_{T,\mu}+\frac{C_{0}(h_1,h_2)}{h_2-h_1}\Expect\left[\sum_{t=\mu_{3k-2,T}}^{\mu_{3k-1,T}-1}\epsilon^{2}_{t}\|\nabla f(\theta_{t})\|^{2}\right],    
\end{align*} that is
\begin{align}\label{fdas123}
\Sigma_{T,\mu}\le \frac{4C_{(h_2-h_1)/8}}{3(h_2-h_1)}  +\frac{4C_0(h_1,h_2)}{3(h_2-h_1)}\Expect\left[\sum_{t=\mu_{3k-2,T}}^{\mu_{3k-1,T}-1}\epsilon^{2}_{t}\|\nabla f(\theta_{t})\|^{2}\right].
\end{align}
For the convenience of the subsequent proof, we define the following quantity, which we refer to as the \( T \)-steps \( m \)-order gradient quadratic variation on $[x,y),$ ($y$ can be taken $+\infty$). Its specific form is as follows:
\begin{align}\label{dasd123213}
\big[\nabla f(\theta_{t})\big]_{T,x,y}^{m}:=\Expect\left[\sum_{t=\mu_{3k-2,T}(x,y)}^{\mu_{3k-1,T}(x,y)-1}\epsilon^{m}_{t}\|\nabla f(\theta_{t})\|^{2}\right].
\end{align}
Now we only need to prove that
\begin{equation}
   \limsup_{T \to +\infty} \left[ \nabla f(\theta_t) \right]_{T,h_1,h_2}^2 < +\infty.\tag{*}
\end{equation}
To do this, we introduce the following lemma. The proof of this lemma (Lemma \ref{iteration}) is rather complex and represents a key challenge in the proof presented in this paper. We have placed the proof in Appendix \ref{P_iteration}.
\begin{lem}\label{iteration}  
Let \( \{ \theta_t \}_{t \geq 1} \subset \mathbb{R}^d \) be the sequence generated by Algorithm \ref{alg:sgd} with the initial point \( \theta_1 \), and assume that the step size sequence \( \{\epsilon_t\}_{t \geq 1} \) satisfies Setting \ref{assump:learning_rate}. Under Assumption \ref{assump:loss_function} (excluding Items (c) and (d)) and Assumption \ref{assump:stochastic_gradient} (excluding Item (d)), the following inequality holds for any \( 0 < a < b < c \), where \( c \) can be taken as \( +\infty \), provided that  
\[
\inf_{\theta\in \{\theta \mid f(\theta) \in [a,c)\}} \|\nabla f(\theta)\| \geq \delta_{a,b} > 0.
\]  
Then, $\forall\ T\ge 1, m\ge 1,$ we have:  
\begin{align}\label{sgd_0123123'}
\big[\nabla f(\theta_{t})\big]_{T,b,c}^{m}&\le C_{1}(a,b) \big[\nabla f(\theta_{t})\big]_{T,a,b}^{m+1}+  C_{2}(m,a,b),
\end{align}  
In above inequality, $\big[\nabla f(\theta_{t})\big]_{T,b,c}^{m}, \ \big[\nabla f(\theta_{t})\big]_{T,a,b}^{m+1}$ defined in Eq. \ref{dasd123213}, and the constants \( C_1(a,b) \) and \( C_2(m,a,b) \) are given by  
\[
C_{1}(a,b) := \left(\left(\frac{3bL}{b-a}+\frac{L}{2}\right)\frac{2G}{\delta_{a,b}^{2}}+\frac{12Lb^{2}G}{(b-a)^{2}}\right)\left(1+\frac{1}{\delta_{a,b}^{2}}\right),
\]
\[
C_{2}(m,a,b) := \frac{6ab}{b-a}\epsilon_{1}^{\frac{m-1}{2}}(\epsilon_{1}^{\frac{m-1}{2}}+\epsilon_{1}^{m-1})C_{\frac{b-a}{2}}.
\]  
\end{lem}  

Based on this recursive lemma, we can further derive the following lemma:
\begin{lem}\label{iteration1}  
Let \( \{ \theta_t \}_{t \geq 1} \subset \mathbb{R}^d \) be the sequence generated by Algorithm \ref{alg:sgd} with the initial point \( \theta_1 \), and assume that the step size sequence \( \{\epsilon_t\}_{t \geq 1} \) satisfies Setting \ref{assump:learning_rate}. Under Assumption \ref{assump:loss_function} (excluding Items (c) and (d)) and Assumption \ref{assump:stochastic_gradient} (excluding Item (d)), the following inequality holds for any \( 0 < x < y \), where \( y \) can be taken as \( +\infty \), provided that  
\[
\inf_{\theta\in \{\theta \mid f(\theta) \in [x,y)\}} \|\nabla f(\theta)\| \geq \delta_{x,y} > 0.
\]  
Then, we have:  
\begin{align}\label{sgd_0123123''}
\limsup_{T\rightarrow+\infty}\big[\nabla f(\theta_{t})\big]_{T,x,y}^{1}<+\infty.
\end{align}  
\end{lem} 
\begin{proof}[Proof of Lemma \ref{iteration1}]
By continuity, it is easy to prove that there exists a constant \( z < x \) such that  
\[
\min_{\theta \in f^{-1}([z, x])} \|\nabla f(\theta)\|\ge \frac{\delta_{x,y}}{2}>0.
\]
Next, we evenly insert \( \lceil p \rceil - 2 \) intermediate points within the interval \( [z, x] \), where the notation \( \lceil \cdot \rceil \) denotes the ceiling function. These \( \lceil p \rceil - 2 \) points can be denoted as  
\[
z := z_{0} < z_{1} < z_{2} < \ldots < h_{3,\lceil p \rceil - 2} < x:=z_{\lceil p \rceil-1}<y := z_{\lceil p \rceil}.
\]  
The points are chosen such that they are equally spaced within the interval \( [z, x] \).

Next, we apply Lemma \ref{iteration}. We choose the parameters in Lemma \ref{iteration} as  
\[
a = z_{i}, \quad b = z_{i+1}, \quad c = z_{i+2}, \quad \text{where } i \in [0, \lceil p \rceil - 2].
\]  

Then, we obtain:
\begin{align*}
\big[\nabla f(\theta_{t})\big]_{T,z_{i+1},z_{i+2}}^{m}&\le C_{1}(z_{i},z_{i+1}) \big[\nabla f(\theta_{t})\big]_{T,z_{i},z_{i+1}}^{m+1}+  C_{2}(m,z_{i},z_{i+1}).   
\end{align*}
Then, we take \( m = \lceil p \rceil -1- i\), and we obtain:
\begin{align*}
\big[\nabla f(\theta_{t})\big]_{T,z_{i+1},z_{i+2}}^{\lceil p \rceil -1 - i}&\le C_{1}(z_{i},z_{i+1}) \big[\nabla f(\theta_{t})\big]_{T,z_{i},z_{i+1}}^{\lceil p \rceil -i}+  C_{2}(\lceil p \rceil -1 - i ,z_{i},z_{i+1}).   
\end{align*}
Then, we perform a backward iteration of the above inequality with respect to the index \( i \), starting from \( i = \lceil p \rceil-2 \) and iterating down to \( i = 0 \). Thus, we obtain:
\begin{align}\label{sgd___0}
    [\nabla f(\theta_{t})]_{T,z_{\lceil p \rceil-1},z_{\lceil p \rceil}}^{1}&\le \left(\prod_{j=0}^{\lceil p \rceil-2}C_{1}(z_{j},z_{j+1})\right)[\nabla f(\theta_{t})]_{T,z_{0},z_{1}}^{\lceil p \rceil}\notag\\&\quad+\sum_{i=0}^{\lceil p \rceil-2}\left(\left(\prod_{j=i}^{\lceil p \rceil-2}C_{1}(z_{j},z_{j+1})\right)C_{2}(\lceil p \rceil-1-j,z_{j},z_{j+1})\right).
\end{align}
It is easy to see that the boundedness of the above expression entirely depends on the boundedness of \( [\nabla f(\theta_{t})]_{T,z_{0},z_{1}}^{\lceil p \rceil}.\) Calculating \([\nabla f(\theta_{t})]_{T,z_{0},z_{1}}^{\lceil p \rceil},\) we obtain:
\begin{align*}
    \limsup_{T\rightarrow+\infty}[\nabla f(\theta_{t})]_{T,z_{0},z_{1}}^{\lceil p \rceil}&= \limsup_{T\rightarrow+\infty}\Expect\left[\sum_{t=\mu_{3k-2,T}(z_0,z_1)}^{\mu_{3k-1,T}(z_0,z_1)-1}\epsilon^{\lceil p \rceil}_{t}\|\nabla f(\theta_{t})\|^{2}\right]\\&\mathop{\le}^{\text{Lemma \ref{loss_bound}}} 2Lh_2\limsup_{T\rightarrow+\infty}\Expect\left[\sum_{t=\mu_{3k-2,T}(z_0,z_1)}^{\mu_{3k-1,T}(z_0,z_1)-1}\epsilon^{\lceil p \rceil}_{t}\right]\\&\mathop{<}^{(**)}2Ly\epsilon_{1}^{\lceil p \rceil-p}\sum_{t=1}^{+\infty}\epsilon^{p}_{t}\\&\mathop{<}^{\text{Setting \ref{assump:learning_rate}}}+\infty.
\end{align*}
In step \((**)\), we used the monotonic decreasing property of the step size sequence \( \{\epsilon_t\}_{t \ge 1} \) as specified in Setting 1. According to Eq. \ref{sgd___0}, we know that \(   \limsup_{T\rightarrow+\infty}[\nabla f(\theta_{t})]_{T,z_{0},z_{1}}^{\lceil p \rceil}<+\infty\) implies \(\limsup_{T\rightarrow+\infty}[\nabla f(\theta_{t})]_{T,z_{\lceil p \rceil-1},z_{\lceil p \rceil}}^{1}<+\infty.\) Noting that \( z_{\lceil p \rceil - 1} = x \) and \( z_{\lceil p \rceil} = y,\) we readily obtain:
\begin{align*}
\limsup_{T\rightarrow+\infty}[\nabla f(\theta_{t})]_{T,x,y}^{1}<+\infty.
\end{align*}
With this, we complete the proof of Lemma \ref{iteration1}.
\end{proof}
Next, by setting \( x = h_1 \) and \( y = h_2 \) in Lemma \ref{iteration1}, we immediately obtain:
\[  \limsup_{T \to +\infty} \left[ \nabla f(\theta_t) \right]_{T,h_1,h_2}^2 \mathop{<}^{(**)} \epsilon_{1}\limsup_{T \to +\infty} \left[ \nabla f(\theta_t) \right]_{T,h_1,h_2}^{1}<+\infty.\tag{*}\]
In step \((**)\), we used the monotonic decreasing property of the step size sequence \( \{\epsilon_t\}_{t \ge 1} \) as specified in Setting 1. Thus, we have proven Equation (*). Substituting the result back into Eq. \ref{fdas123}, we obtain  
\[
\limsup_{T\rightarrow+\infty}\Sigma_{T,\mu} < +\infty.
\]  
Furthermore, this implies that  
\[
\lim_{T \rightarrow +\infty} U_{H_x, T} < +\infty \quad \text{a.s.}
\]  
This concludes the proof of Statement (b). 
In summary, we have completed the proofs of both Statement (a) and Statement (b), thereby establishing the theorem.
\end{proof}
In fact, when the loss function possesses certain special structures—such as the critical point set having only finitely many connected components—the almost sure convergence of the function values to a critical value can already guarantee that the gradient norm converges to zero almost surely. However, for general loss functions, this guarantee does not hold. To address this, we need to impose an additional condition: the boundedness of the \( (2p - 2) \)-th moment of the stochastic gradient within a local neighborhood (Assumption \ref{assump:stochastic_gradient}~Item (d)). This will be the focus of the next subsection.

\subsection{Asymptotic Convergence of The Gradient Norm}
In this subsection, we primarily focus on studying the asymptotic convergence of the gradient norm. We divide this into two subsections: asymptotic almost sure convergence and asymptotic \( L_2 \) convergence.
\subsubsection{Asymptotic Almost Sure Convergence}
Note that the result in Theorem \ref{thm:as_convergence:-1}, where the loss function value almost surely converges to a critical value, does not guarantee that the gradient norm also almost surely converges to 0. We need to provide a condition for the boundedness of the random gradient's \( 2p-2 \)th moment within a small range, as specified in Assumption \ref{assump:stochastic_gradient}~Item (d).

Next, we present the theorem on the convergence of the gradient norm.

\begin{thm}\label{thm:as_convergence}
Let \( \{ \theta_t \}_{t \ge 1} \subset \mathbb{R}^d \) be the sequence generated by Algorithm \ref{alg:sgd} with the initial point \( \theta_1 \). Suppose the step sizes \( \{\epsilon_t\}_{t \ge 1} \) satisfy the conditions in Setting \ref{assump:learning_rate}, and Assumptions \ref{assump:loss_function} and \ref{assump:stochastic_gradient} hold. Then, \[\lim_{t\rightarrow+\infty}\|\nabla f(\theta_{t})\|= 0 \quad \text{a.s.}\]
\end{thm}
This theorem shows that the gradients evaluated at the iterates converge to zero almost surely, indicating that the algorithm approaches a critical point of the loss function along almost every trajectory. It is worth mentioning that our proof only assume the local boundedness of the \(2p{-}2\)-th moment within a region whose image under the mapping \(f\) has Lebesgue measure controlled by a constant \(\delta > 0\), which can be made arbitrarily small. In contrast, previous works typically require global boundedness of the $2p-2$-th moment.

The proof of this theorem requires the use of the ODE method, which is commonly applied in stochastic approximation. First, we present the fundamental convergence theorem of the ODE method. Specifically, we have:
\begin{proposition}\label{SA_p}
Let $F$ be a continuous globally integrable vector field. Assume that
\begin{enumerate}[label=\textnormal{(A.\arabic*)},leftmargin=*]
    \item Suppose $\sup_{t\ge 1} \|\theta_t\|< \infty\ \ \text{a.s.}$ 
    \item For all $T > 0$
    \[
    \lim_{t \to \infty} \sup \left\lbrace \left\lVert \sum_{i=t}^{k} \epsilon_{i}\left(\nabla f(\theta_{t})-g_t\right) \right\lVert : k = t, \dots, m(\Sigma_{\epsilon}(t) + T) \right\rbrace = 0,
    \]
    where 
    \[\Sigma_{\epsilon}(t):=\sum_{k=1}^{t}\epsilon_{k}\ \ \text{and}\ \ m(t):=\max\{j\ge 0:  \Sigma_{\epsilon}(j)\le t\}.\] 
    \item \(f(V)\) is nowhere dense on \(\mathbb{R}\), where \(V\) is the fixed point set of the ODE: \(\dot{\theta} = -f(\theta)\).

\end{enumerate}
Then all limit points of the sequence \(\{\theta_{t}\}_{t\ge 1}\) are fixed points of the ODE: \(\dot{\theta} = -f(\theta)\).
\end{proposition}
\begin{rem}
Property \ref{SA_p} synthesizes results from Proposition 4.1, Theorem 5.7, and Proposition 6.4 in \cite{benaim2006dynamics}. Proposition 4.1 shows that the trajectory of an algorithm satisfying Property \ref{SA_p} Item (A.1) and (A.2) forms a precompact asymptotic pseudotrajectory of the corresponding ODE system. Meanwhile, Theorem 5.7 and Proposition 6.4 demonstrate that all limit points of this precompact asymptotic pseudotrajectory are fixed points of the ODE system. 
\end{rem}
Next, we proceed to the formal proof.
\begin{proof}
The proof strategy is very clear. We only need to demonstrate that SGD, under our assumptions (i.e., Assumption \ref{assump:loss_function}$\sim$\ref{assump:stochastic_gradient}), satisfies the three items in Property \ref{SA_p}. Next, we proceed with the verification.

\noindent{\bf Phase I:} (Proving Item (A.1))
According to Theorem \ref{thm:as_convergence:-1}, we can easily obtain
\[
\sup_{t \ge 1} f(\theta_{t}) < +\infty \quad \text{a.s.}
\]
Furthermore, by applying Item (b) of Assumption \ref{assump:loss_function}, i.e., coercivity, we can immediately derive the following boundedness result for the iterates, i.e.,
\[\sup_{t \ge 1} \|\theta_t\| < +\infty \quad \text{a.s.}\]
This confirms Item (A.1).

\noindent{\bf Phase II:} (Proving Item (A.2))
We consider the following expression:
\[\Theta_{t}:= \sup \left\lbrace \left\lVert \sum_{i=t}^{k}\I_{[\theta_{i}\in S_{\delta}]}\cdot \epsilon_{i}\left(\nabla f(\theta_{i})-g_i\right) \right\lVert^{2p-2} : k = t, \dots, m(\Sigma_{\epsilon}(t) + T) \right\rbrace.\]
We aim to prove that 
\[
\lim_{t \rightarrow +\infty} \Theta_t = 0 \quad \text{a.s.}
\]  
It is easy to see that, in order to prove this, it suffices to show that  
\[
\lim_{k \rightarrow +\infty} \Theta_{m(kT)} = 0 \quad \text{a.s.} \quad (k \in \mathbb{N}_+).
\]
We observe that the sequence \( \left\{\I_{[\theta_{t}\in S_{\delta}]}\cdot \epsilon_{t}\left(\nabla f(\theta_{t})-g_t\right), \mathscr{F}_{t}\right\}_{t \ge 1} \) forms a martingale difference sequence. Therefore, for the expression in Item (A.2), we can apply the following \emph{Burkholder's inequality}\footnote{\textbf{Burkholder's Inequality}: 
Let \( \{(M_n, \mathscr{F}_n)\}_{n \ge 1} \) be a martingale with \( M_1 = 0 \) almost surely. For any \( 1 \leq p < \infty \), there exist constants \( c_p > 0 \) and \( C_p > 0 \), depending only on \( p \), such that for all \( t \ge 1 \),
\[
c_p \, \mathbb{E}[(S(M_t))^p] \leq \mathbb{E}[(M_t^*)^p] \leq C_p \, \mathbb{E}[(S(M_t))^p],
\]
where \( M_t^* = \sup_{1 \le n \le t} |M_n| \) and  
\[
S(M_t) = \left( \sum_{i=1}^{t} (M_i - M_{i-1})^2 \right)^{1/2}.
\]}, that is,

\begin{align*}
\mathbb{E} \left[ \Theta_{m(kT)} \right]
&\overset{\text{\emph{Burkholder's inequality}}}{\le} C_{2p-2} 
\mathbb{E} \left[ \left( \sum_{i=m(kT)}^{m((k+1)T)} 
    \mathbb{I}_{[\theta_i \in S_{\delta}]} \cdot \epsilon_i^2 \| \nabla f(\theta_i) - g_i \|^2 \right)^{p-1} \right] \notag \\
\overset{\text{\emph{Holder's inequality}}}{\le}& C_p 
\mathbb{E} \left[ \left( \left( \sum_{i=m(kT)}^{m((k+1)T)} \epsilon_i \right)^{p-2} \right)
    \sum_{i=m(kT)}^{m((k+1)T)} \mathbb{I}_{[\theta_i \in S_{\delta}]} \cdot \epsilon_i^p \| \nabla f(\theta_i) - g_i \|^{2p-2} \right] \notag \\
\le C_p T^{p-2} \mathbb{E}& \left[ \sum_{i=m(kT)}^{m((k+1)T)} 
    \mathbb{I}_{[\theta_i \in S_{\delta}]} \epsilon_i^p \| \nabla f(\theta_i) - g_i \|^{2p-2} \right] \notag \\
\overset{(*)}{\le} 2^{2p-3} C_p &T^{p-2} \left( (2L(D_{\eta} + \delta))^{p-1} + M_1^{2p-2} \right) 
    \sum_{i=m(kT)}^{m((k+1)T)} \epsilon_i^p.
\end{align*}

In the above derivation, step \((*)\) follows from Assumption \ref{assump:stochastic_gradient} Item (d), which ensures that when \( \theta_t \in S_{\delta} \), the \( (2p - 2) \)th moment is bounded. Specifically, we have:
\begin{align*}
    &\quad\Expect\left[\sum_{i=m(kT)}^{m((k+1)T)}\I_{[\theta_{i}\in S_{\delta}]}\epsilon_{i}^{p}\left\|\nabla f(\theta_{i})-g_{i}\right\|^{2p-2}\right]\\&=\Expect\left[\sum_{i=m(kT)}^{m((k+1)T)}\I_{[\theta_{i}\in S_{\delta}]}\epsilon_{i}^{p}\Expect\left[\left\|\nabla f(\theta_{i})-g_{i}\right\|^{2p-2}|\mathscr{F}_{t-1}\right]\right]\\&\le 2^{2p-3}\left((2L(D_{\eta}+\delta))^{p-1}+M_1^{2p-2}\right)\sum_{i=m(kT)}^{m((k+1)T)}\epsilon_{i}^{p}.
\end{align*}
Therefore, we can easily obtain:
\begin{align*}
    \sum_{k=1}^{+\infty}\Expect\left[\Theta_{m(kT)}\right]&\le 2^{2p-3}\left((2L(D_{\eta}+\delta))^{p-1}+M_1^{2p-2}\right)\sum_{i=1}^{+\infty}\epsilon_{i}^{p}\\&\mathop{<}^{\text{Setting 1}}+\infty.
\end{align*}
Then, by the \emph{Lebesgue's Monotone Convergence Theorem}, we can readily obtain:
\begin{align*}
    \sum_{k=1}^{+\infty}\Theta_{m(kT)}<+\infty\ \ \text{a.s.},
\end{align*}
which implies,
\begin{align*}
   \lim_{k\rightarrow+\infty}\Theta_{m(kT)}=0\ \ \text{a.s.} 
\end{align*}
Furthermore, we can conclude that  
\begin{align}\label{Theta_t}
\lim_{t \rightarrow +\infty} \Theta_t = 0 \quad \text{a.s.}
\end{align}
Next, according to Theorem \ref{thm:as_convergence:-1}, we know that the sequence of loss function values \( \{f(\theta_t)\}_{t \ge 1} \) converges to some critical value almost surely. This implies that there exists a finite time \( T_0 < +\infty \ \ \text{a.s.} \) (note that \( T_0 \) is a random variable and may vary across different trajectory), such that for all \( t > T_0 \), we have \( f(\theta_t) \in S_{\delta} \).  

This, in turn, means that for all \( t > T_0 \),  
\[
\mathbb{I}_{[f(\theta_t) \in S_{\delta}]} = \mathbf{1}.
\]  

Consequently, for all \( t > T_0 \), we also have:
\begin{align*}
    \sup \left\lbrace \left\lVert \sum_{i=t}^{k}\epsilon_{i}\left(\nabla f(\theta_{i})-g_i\right) \right\lVert^{2p-2} : k = t, \dots, m(\Sigma_{\epsilon}(t) + T) \right\rbrace=\Theta_{t}.
\end{align*}
Then due to \[\lim_{t\rightarrow+\infty}\Theta_{t}=0\ \ \text{Eq. \ref{Theta_t}},\] we get that:
\begin{align*}
   \lim_{t\rightarrow+\infty} \sup \left\lbrace \left\lVert \sum_{i=t}^{k}\epsilon_{i}\left(\nabla f(\theta_{i})-g_i\right) \right\lVert^{2p-2} : k = t, \dots, m(\Sigma_{\epsilon}(t) + T) \right\rbrace=0\ \ \text{a.s.},
\end{align*}
that is,
\begin{align*}
   \lim_{t\rightarrow+\infty} \sup \left\lbrace \left\lVert \sum_{i=t}^{k}\epsilon_{i}\left(\nabla f(\theta_{i})-g_i\right) \right\lVert : k = t, \dots, m(\Sigma_{\epsilon}(t) + T) \right\rbrace=0\ \ \text{a.s.}
\end{align*}
Thus, we have completed the proof of Item (A.2) in Proposition \ref{SA_p}. 

\noindent{\bf Phase III:} (Proving Item (A.3))
We know that for the dynamical system \(\dot{\theta} = -f(\theta),\) its set of fixed points \( V \) coincides with the set of critical points \( \text{Cric}(f)\) of the loss function \( f.\)

Since \( f \) is \( d \)-times differentiable, by Sard's theorem \cite{sard1942measure, bates1993toward}, we know that the Lebesgue measure of \( f(\text{Crit}(f)) \) is zero, i.e., \( m(f(\text{Crit}(f))) = 0 \), where \( m(\cdot) \) denotes Lebesgue measure and \( f(\text{Crit}(f)) \) represents the image of \( \text{Crit}(f) \) under \( f .\) Furthermore, by the corecivity assumption (Assumption \ref{assump:loss_function} (c)), \( f(\text{Crit}(f)) \) is compact. Therefore, \( f(\text{Crit}(f)) \) is nowhere dense in $\mathbb{R}.$ This verifies Item (A.3) in Proposition \ref{SA_p}.

In summary, we have verified all three conditions in Proposition \ref{SA_p}. Therefore, by Proposition \ref{SA_p}, we can conclude that  
\[
\lim_{t \rightarrow +\infty} \| \nabla f(\theta_t) \| = 0 \quad \text{a.s.}
\]  
This completes the proof.
\end{proof}
Next, we introduce the result on asymptotic \( L_2 \) convergence.

\subsubsection{\texorpdfstring{Asymptotic \(L_2\) Convergence}{Asymptotic L2 Convergence}}
We need to point out that almost sure convergence does not imply \( L_2 \) convergence. Specifically, consider the following counterexample:
\begin{counter}\textbf{(Almost Sure vs. $L_{2}$ Convergence)}\label{as_vs_L_2}
Consider a sequence of random variables $\{\zeta_{t}\}_{t\ge 1}$ where $\Pro(\zeta_{t}=0)=1-1/n^{2}$ and $\Pro(\zeta_{t}=n)=1/n^{2}$. According to the \textit{Borel-Cantelli lemma}, we have $\lim_{t\rightarrow\infty}\zeta_{t}=0\ \ \text{a.s.}$ However, it can be shown that $\Expect[|\zeta_{t}|^{2}]=1$ for all $n > 0$.
\end{counter}

Next, we present the \( L_2 \) convergence result.
\begin{thm}\label{thm:as_convergence_1}
Let \( \{ \theta_t \}_{t \ge 1} \subset \mathbb{R}^d \) be the sequence generated by Algorithm \ref{alg:sgd} with the initial point \( \theta_1 \). Suppose the step sizes \( \{\epsilon_t\}_{t \ge 1} \) satisfy the conditions in Setting \ref{assump:learning_rate}, and Assumptions \ref{assump:loss_function} and \ref{assump:stochastic_gradient} hold. Then, \[\lim_{t\rightarrow+\infty}\Expect\left[\|\nabla f(\theta_{t})\|^{2}\right]= 0.\]
\end{thm}
This theorem differs in focus from the almost sure convergence result in Theorem \ref{thm:as_convergence}. It shows that the convergence of gradient norm across different trajectories is uniform in the sense of the \( L_2 \)-norm \footnote{For a random variable \( X \), the \( L_2 \)-norm is defined as \( \|X\|_{L_2} = \sqrt{\mathbb{E}[\|X\|^2] } \), provided \( \mathbb{E}[\|X\|^2] < \infty \).} with respect to the underlying randomness.

Next, we present the proof of this theorem.

\begin{proof}
Since we have already established almost sure convergence in Theorem~\ref{thm:as_convergence}, it follows from the \emph{Lebesgue's Dominated Convergence Theorem} that to establish \( L_2 \) convergence, it suffices to show
\begin{equation} \label{eq:sharp}
\mathbb{E} \left[ \sup_{t \ge 1} \|\nabla f(\theta_t)\|^2 \right] < +\infty \tag{\#}
\end{equation}
We now focus on proving inequality~(\ref{eq:sharp}).

We define the following constant upper bound:  
\[
M := \max\{f(\theta_1) - f^*, D_{\eta}\}.
\]  
Then, for any fixed time \( T \ge 1 \), we consider the quantity \(\Expect\left[ \sup_{1 \le t \le T} (f(\theta_t) - f^*)\right].\) We denote  
\[
t_T^* := \arg\max_{1 \le t \le T} \{f(\theta_t) - f^*\} \quad \text{(if the argmax is not unique, we take the smallest such } t).
\] 
Then we have the following inequality:
\begin{align*}
&\quad\Expect\left[\sup_{1 \le t \le T} (f(\theta_t) - f^*)\right]\\&=\underbrace{\Expect\left[\I\left[t_T^*\in \bigcup_{k\ge 1}[\mu_{2k,T}(M,+\infty),\mu_{2k+1,T}(M,+\infty))\right]\cdot\sup_{1 \le t \le T} (f(\theta_t) - f^*)\right]}_{\text{SUP}_{T,1}}\\&\quad+\underbrace{\Expect\left[\I\left[t_T^*\in \bigcup_{k\ge 1}[\mu_{2k-1,T}(M,+\infty),\mu_{2k,T}(M,+\infty))\right]\cdot\sup_{1 \le t \le T} (f(\theta_t) - f^*)\right]}_{\text{SUP}_{T,2}}.
\end{align*}
The definition of the stopping time sequence \( \{\mu_{n,T}(M, +\infty)\}_{n \ge 1} \) can be found in Eq.~\ref{qwerewqwq} (by substituting \( h_1 = M \) and \( h_2 = +\infty \) into Eq.~\ref{qwerewqwq}). Similarly, to keep the notation concise, we will simplify \( \mu_{n,T}(M, +\infty) \) and \( \mu_n(M, +\infty) \) to \( \mu_{n,T} \) and \( \mu_n \), respectively, for the remainder of this proof. For convenient, we assign
\begin{align*}
    &\I_{1}:=\I\left[t_T^*\in \bigcup_{k\ge 1}[\mu_{2k,T}(M,+\infty),\mu_{2k+1,T}(M,+\infty))\right],\\&\I_{2}:=\I\left[t_T^*\in \bigcup_{k\ge 1}[\mu_{2k-1,T}(M,+\infty),\mu_{2k,T}(M,+\infty))\right].
\end{align*}

We now proceed to bound \( \text{SUP}_{T,1} \) and \( \text{SUP}_{T,2} \) separately. For \( \text{SUP}_{T,1} \), it is clear from the definition of the stopping time that  
\[
\text{SUP}_{T,1} \le M,
\]  which means
\begin{align}\label{SUP_T,1}
\limsup_{T\rightarrow+\infty}\text{SUP}_{T,1} \le M.
\end{align}

We now focus on bounding \( \text{SUP}_{T,2} \). Since the indicator function 
\(\mathbb{I}_2\) ensures that \( t_T^* \) must lie in some interval of the form
\[
[\mu_{2k-1,T}, \mu_{2k,T}), \footnote{When \( \mu_{2k-1,T} = \mu_{2k,T} \), we define the interval 
as empty: \( [\mu_{2k-1,T}, \mu_{2k,T}) = \emptyset \).}\]
we denote this specific interval by
\[
[\mu_{2k^*-1,T}, \mu_{2k^*,T}).
\]

Next, we have:

\begin{align}\label{sup_0}
    \text{SUP}_{T,2}&\le \Expect\left[f(\theta_{\mu_{2k^*-1,T}-1})-f^*\right]+\Expect\left[\I_{1}(f(\theta_{t_T^*})-f(\theta_{\mu_{2k^*-1,T}-1}))\right]\notag\\&\le M+\Expect\left[\I_{1}(f(\theta_{t_T^*})-f(\theta_{\mu_{2k^*-1,T}-1}))\right]\notag\\&=M+\underbrace{\Expect\left[\I_{1}(f(\theta_{t_T^*})-f(\theta_{\mu_{2k^*-1,T}}))\right]}_{\text{SUP}_{T,2,1}}+\underbrace{\Expect\left[\I_{1}(f(\theta_{\mu_{2k^*-1,T}})-f(\theta_{\mu_{2k^*-1,T}-1}))\right]}_{\text{SUP}_{T,2,2}}.
\end{align}
For \( \text{SUP}_{T,2,1} \), we have:
\begin{align*}
\text{SUP}_{T,2,1}&  =  \Expect\left[\I_{1}(f(\theta_{t_T^*})-f(\theta_{\mu_{2k^*-1,T}}))\right]\notag\\&=\Expect\left[\I_{1}\sum_{t=\mu_{2k^*-1,T}}^{t_T^*-1}\left(f(\theta_{t+1})-f(\theta_{t})\right)\right]\notag\\&\mathop{\le}^{\text{Lemma \ref{descent_lemma}}}\Expect\left[\I_{1}\sum_{t=\mu_{2k^*-1,T}}^{t_T^*-1}\left(-\epsilon_{t}\nabla f(\theta_{t})^{\top}g_{t}+\frac{L\epsilon_{t}^{2}}{2}\|g_{t}\|^{2}\right)\right]\notag\\&\le \Expect\left[\I_{1}\sum_{t=\mu_{2k^*-1,T}}^{\mu_{2k^*,T}-1}\epsilon_{t}\|\nabla f(\theta_{t})\|\|g_{t}\|\right]+\frac{L}{2}\Expect\left[\I_{1}\sum_{t=\mu_{2k^*-1,T}}^{\mu_{2k^*,T}-1}\epsilon_{t}^{2}\|g_{t}\|^{2}\right]\notag\\&\le \Expect\left[\sum_{t=\mu_{2k^*-1,T}}^{\mu_{2k^*,T}-1}\epsilon_{t}\|\nabla f(\theta_{t})\|\|g_{t}\|\right]+\frac{L}{2}\Expect\left[\sum_{t=\mu_{2k^*-1,T}}^{\mu_{2k^*,T}-1}\epsilon_{t}^{2}\|g_{t}\|^{2}\right]\notag\\&\le \Expect\left[\sum_{k=1}^{+\infty}\sum_{t=\mu_{2k-1,T}}^{\mu_{2k,T}-1}\epsilon_{t}\|\nabla f(\theta_{t})\|\|g_{t}\|\right]+\frac{L}{2}\Expect\left[\sum_{k=1}^{+\infty}\sum_{t=\mu_{2k-1,T}}^{\mu_{2k,T}-1}\epsilon_{t}^{2}\|g_{t}\|^{2}\right]\notag\\&\mathop{=}^{\text{\emph{Doob's Stopped Theorem}}}\Expect\left[\sum_{k=1}^{+\infty}\sum_{t=\mu_{2k-1,T}}^{\mu_{2k,T}-1}\epsilon_{t}\|\nabla f(\theta_{t})\|\Expect\left[\|g_{t}\||\mathscr{F}_{t-1}\right]\right]\notag\\&\quad+\frac{L}{2}\Expect\left[\sum_{k=1}^{+\infty}\sum_{t=\mu_{2k-1,T}}^{\mu_{2k,T}-1}\epsilon_{t}^{2}\Expect\left[\|g_{t}\|^{2}|\mathscr{F}_{t-1}\right]\right]\notag\\&\mathop{\le}^{(**)}\left(\sqrt{G}\left(1+\frac{1}{\eta}\right)+\epsilon_{1}G\left(1+\frac{1}{\eta^{2}}\right)\right)\left[\nabla f(\theta_{t})\right]_{T,M,+\infty}^{1}.
\end{align*}
In step \((**)\), we need to note that when \( \mu_{2k-1,T} < \mu_{2k,T} \), it must hold that \( \mu_{2k-1} < T \). According to the definition of the stopping time, this implies \( f(\theta_t) - f^* \ge M \ge D_{\eta}\ (\forall \ t\in[\mu_{2k-1,T},\mu_{2k,T}))\) By Assumption \ref{assump:loss_function} Item (d), it then follows that \( \|\nabla f(\theta_t)\| \ge \eta\ (\forall \ t\in[\mu_{2k-1,T},\mu_{2k,T})) \). Then we have:
\begin{align}\label{tuidao}
&\Expect\left[\|g_{t}\||\mathscr{F}_{t-1}\right]\mathop{\le}^{\text{Assumption \ref{assump:stochastic_gradient} Item (b)}} \sqrt{G}\left(\|\nabla f(\theta_{t})\|+1\right)\le \sqrt{G}\left(1+\frac{1}{\eta}\right),\notag
\\&\Expect\left[\|g_{t}\|^{2}|\mathscr{F}_{t-1}\right]\mathop{\le}^{\text{Assumption \ref{assump:stochastic_gradient} Item (b)}} {G}\left(\|\nabla f(\theta_{t})\|^{2}+1\right)\le {G}\left(1+\frac{1}{\eta^{2}}\right).
\end{align}
Therefore, based on Lemma \ref{iteration1}, we can readily obtain:  
\begin{align} \label{SUP_T,2,1}
\limsup_{T \rightarrow +\infty} \text{SUP}_{T,2,1}&\le\left(\sqrt{G}\left(1+\frac{1}{\eta}\right)+\epsilon_{1}G\left(1+\frac{1}{\eta^{2}}\right)\right)\limsup_{T\rightarrow+\infty}\left[\nabla f(\theta_{t})\right]_{T,M,+\infty}^{1}\notag\\&\mathop{<}^{\text{Lemma \ref{iteration1}}} +\infty.
\end{align}
For \( \text{SUP}_{T,2,2} \), we further decompose it into \( \text{SUP}_{T,2,2,1} \) and \( \text{SUP}_{T,2,2,2} \) for detailed analysis, as follows:
\begin{align*}
\text{SUP}_{T,2,2}&=\underbrace{\Expect\left[\I_{1}\I_{[f(\theta_{\mu_{2k^*-1,T}-1})-f^< D_{\eta} ]}(f(\theta_{\mu_{2k^*-1,T}})-f(\theta_{\mu_{2k^*-1,T}-1}))\right]}_{\text{SUP}_{T,2,2,1}}\\&\quad+\underbrace{\Expect\left[\I_{1}\I_{[f(\theta_{\mu_{2k^*-1,T}-1})-f^]\ge D_{\eta} ]}(f(\theta_{\mu_{2k^*-1,T}})-f(\theta_{\mu_{2k^*-1,T}-1}))\right]}_{\text{SUP}_{T,2,2,2}}.
\end{align*}
For $\text{SUP}_{T,2,2,1},$ we have:
\begin{align*}
\text{SUP}_{T,2,2,1}&=\Expect\left[\I_{1}\I_{[f(\theta_{\mu_{2k^*-1,T}-1})-f^< D_{\eta} ]}(f(\theta_{\mu_{2k^*-1,T}})-f(\theta_{\mu_{2k^*-1,T}-1})-1)\right]  \\&\quad+\Expect\left[\I_{1}\I_{[f(\theta_{\mu_{2k^*-1,T}-1})-f^< D_{\eta} ]}\right]\\&\le   \Expect\left[\I_{1}\I_{[f(\theta_{\mu_{2k^*-1,T}-1})-f^< D_{\eta} ]}(f(\theta_{\mu_{2k^*-1,T}})-f(\theta_{\mu_{2k^*-1,T}-1})-1)_{+}\right]  +1\\&\le \sum_{t=1}^{+\infty}\Expect\left[\overline{\Delta}_{t,1}\right]+1\\&\mathop{\le}^{\text{Lemma \ref{lem_1123}}} C_1 +1,
\end{align*}
which means
\begin{align}\label{SUP_T,2,2,1}
\limsup_{T\rightarrow+\infty}\text{SUP}_{T,2,2,1}\le C_1 +1.
\end{align}
For $\text{SUP}_{T,2,2,2},$ we have:

\begin{align}\label{SUP_T,2,2,2}
&\limsup_{T\rightarrow+\infty}\text{SUP}_{T,2,2,2}
\mathop{\le}^{\text{Lemma \ref{descent_lemma}}}
\limsup_{T\rightarrow+\infty}
\mathbb{E}\Bigg[
    \mathbb{I}_{1}
    \mathbb{I}_{[f(\theta_{\mu_{2k^*-1,T}-1}) - f^* \ge D_{\eta}]}
    \Bigg(
        \frac{L \epsilon_{\mu_{2k^*-1,T}-1}^{2}}{2}
        \|g_{\mu_{2k^*-1,T}-1}\|^{2} \notag\\
&\quad +
        \epsilon_{\mu_{2k^*-1,T}-1}
        \|\nabla f(\theta_{\mu_{2k^*-1,T}-1})\|
        \|g_{\mu_{2k^*-1,T}-1}\|
    \Bigg)
\Bigg] \notag\\
&\le \frac{L}{2} \limsup_{T\rightarrow+\infty} \sum_{t=1}^{T}
\mathbb{E}\left[
    \mathbb{I}_{[f(\theta_t) - f^* \ge D_{\eta}]} \epsilon_t^2 \|g_t\|^2
\right] \notag\\
&\quad +
\limsup_{T\rightarrow+\infty} \sum_{t=1}^{T}
\mathbb{E}\left[
    \mathbb{I}_{[f(\theta_t) - f^* \ge D_{\eta}]} \epsilon_t
    \|\nabla f(\theta_t)\| \|g_t\|
\right] \notag\\
&\mathop{\le}^{(**)}
\left(
    \frac{L G \epsilon_1}{2}\left(1 + \frac{1}{\eta^2}\right)
    + \sqrt{G}\left(1 + \frac{1}{\eta}\right)
\right)
\limsup_{T\rightarrow+\infty}
\sum_{t=1}^{T}
\mathbb{E}\left[
    \mathbb{I}_{[f(\theta_t) - f^* \ge D_{\eta}]}
    \epsilon_t \|\nabla f(\theta_t)\|^2
\right] \notag\\
&=
\left(
    \frac{L G \epsilon_1}{2}\left(1 + \frac{1}{\eta^2}\right)
    + \sqrt{G}\left(1 + \frac{1}{\eta}\right)
\right)
\limsup_{T\rightarrow+\infty}
\sum_{k=1}^{+\infty}
\mathbb{E}\left[
    \sum_{t=\mu_{2k-1,T(D_{\eta},+\infty)}}^{\mu_{2k,T}(D_{\eta},+\infty)}
    \epsilon_t \|\nabla f(\theta_t)\|^2
\right] \notag\\
&=
\left(
    \frac{L G \epsilon_1}{2}\left(1 + \frac{1}{\eta^2}\right)
    + \sqrt{G}\left(1 + \frac{1}{\eta}\right)
\right)
\limsup_{T\rightarrow+\infty}
\left[\nabla f(\theta_t)\right]_{T, D_{\eta}, +\infty}^{1} \notag\\
&< +\infty.
\end{align}

In the above derivation, for step \((**)\), we need to note that when the indicator function \( \mathbb{I}_{\{f(\theta_t) - f^* \ge D_{\eta}\}} = 1 \), it follows from Assumption \ref{assump:loss_function}, Item (d), that \( \|\nabla f(\theta_t)\| \ge \eta \). Therefore, we can carry out this step using an argument similar to that in Eq.~\ref{tuidao}.

By combining the bounds established in Eq.~\ref{SUP_T,2,2,1} and Eq.~\ref{SUP_T,2,2,2} for \( \limsup_{T \rightarrow +\infty} \text{SUP}_{T,2,2,1} \) and \( \limsup_{T \rightarrow +\infty} \text{SUP}_{T,2,2,2} \), respectively, we deduce the following:  
\[
\limsup_{T \rightarrow +\infty} \text{SUP}_{T,2,2} < +\infty.
\]  
Subsequently, invoking the result on \( \text{SUP}_{T,2,1} \) derived in Eq.~\ref{SUP_T,2,1}, we conclude that  
\[
\limsup_{T \rightarrow +\infty} \text{SUP}_{T,2} < +\infty.
\]  
Finally, in conjunction with the upper bound on \( \text{SUP}_{T,1} \) provided in Eq.~\ref{SUP_T,1}, we arrive at  
\[
\limsup_{T \rightarrow +\infty} \Expect\left[\sup_{1\le t\le T}\left(f(\theta_{t})-f^*\right)\right] < +\infty.
\]
Note that the sequence \( \left\{ \sup_{1 \le t \le T} \left( f(\theta_t) - f^* \right) \right\}_{T \ge 1} \) is monotonically increasing. Therefore, by the \emph{Lebesgue's Monotone Convergence Theorem}, we can exchange the limit superior with the expectation. Specifically, we have:  
\begin{align*}
\mathbb{E} \left[ \sup_{ t \ge 1} \left( f(\theta_t) - f^* \right) \right] &=
\mathbb{E} \left[ \limsup_{T \rightarrow +\infty} \sup_{1 \le t \le T} \left( f(\theta_t) - f^* \right) \right] 
= \limsup_{T \rightarrow +\infty} \mathbb{E} \left[ \sup_{1 \le t \le T} \left( f(\theta_t) - f^* \right) \right]\\&<+\infty.
\end{align*}
Then, applying Lemma \ref{loss_bound}, we immediately obtain:  
\[
 \mathbb{E} \left[ \sup_{t \ge 1} \|\nabla f(\theta_{t})\|^{2} \right]\mathop{\le}^{\text{Lemma \ref{loss_bound}}} 2L\mathbb{E} \left[ \sup_{ t \ge 1} \left( f(\theta_t) - f^* \right) \right] < +\infty.
\]
This establishes inequality~(\ref{eq:sharp}). Then, by the \emph{Lebesgue's Dominated Convergence Theorem}, we can immediately deduce the \( L_2 \) convergence result from the almost sure convergence established in Theorem~\ref{thm:as_convergence}, namely,  
\[
\lim_{t \rightarrow +\infty} \mathbb{E} \left[ \|\nabla f(\theta_t)\|^2 \right] = 0.
\]
With this, we complete the proof.

\end{proof}

\section{Overall conclusions}

In this article, we employ a novel analytical method, called stopping time method, to explore the asymptotic convergence of the SGD algorithm under more relaxed step-size conditions, providing more step-size options with convergence guarantees for practical applications. This work is distinguished by its minimal set of required assumptions, thereby broadening the scope of SGD applications to practical scenarios where traditional assumptions may not apply. The underlying philosophy of the stopping time method could potentially serve as a template for proving the convergence of other related stochastic optimization algorithms, such as Adaptive Moment Estimation (ADAM) \cite{kingma2014adam} and Stochastic Gradient Descent with Momentum (SGDM) \cite{gitman2019understanding}, etc.
\\

\bibliography{ref}
\vskip 0.2in

\appendix
\section{Supporting Lemmas}
\begin{lem}\label{loss_bound} 
Suppose that $f(x)$ is differentiable and lower bounded $ f^{\ast} = \inf_{x\in \ \mathbb{R}^{d}}f(x) >-\infty$ and $\nabla f(x)$ is Lipschitz continuous with parameter $\mathcal{L} > 0$, then $\forall \ x\in \ \mathbb{R}^{d}$, we have
\begin{align*}
\big\|\nabla f(x)\big\|^{2}\le {2\mathcal{L}}\big(f(x)-f^{*}\big).
\end{align*}
\end{lem}

\begin{lem}[Descent Lemma]\label{descent_lemma}
Let \(\{\theta_t\}\) be the sequence generated by the SGD. Under Assumption \ref{assump:loss_function}.1, the following inequality holds for all \(t\ge 1\):
\begin{align*}
    f(\theta_{t+1})-f(\theta_{t})\le-\epsilon_{t}\|\nabla f(\theta_{t})\|^{2}+\frac{L\epsilon_{t}^{2}}{2}\|g_t\|^{2}+M_{t},
\end{align*}
where $M_{t}:=\epsilon_{t}\nabla f(\theta_{t})^{\top}(\nabla f(\theta_{t})-g_{t}).$
\end{lem}

\begin{lem}
\label{lem:descent_prime}
Let \( \{ \theta_t \}_{t \geq 1} \subset \mathbb{R}^d \) be the sequence generated by SGD with step sizes \( \{ \epsilon_t \}_{t \geq 1} \). Under Assumption \ref{assump:loss_function}\ref{assump:loss_function_lipschitz} (Lipschitz Continuous Gradient), and the assumptions on the stochastic gradient (Assumption \ref{assump:stochastic_gradient}), the following inequality holds for all integers \( n \geq 1 \), all time indices \( t \geq 0 \), and for all \( y > 0 \):
\begin{align}\label{eq:sgd_2}
    \mathbb{I}_{ \| \nabla f(\theta_t) \|^2 \ge  y } \cdot \epsilon_t^{m} \| \nabla f(\theta_t) \|^2 &\leq \mathbb{I}_{ \| \nabla f(\theta_t) \|^2 \ge  y } \cdot \epsilon_t^{m-1} \Delta_{f_t} + \mathbb{I}_{ \| \nabla f(\theta_t) \|^2 \ge  y } \cdot \epsilon_t^{m-1} \left( M_{t,1} + M_{t,2} \right) \notag\\
    &\quad + \frac{L G}{2} \left( 1 + \frac{1}{y} \right) \mathbb{I}_{ \| \nabla f(\theta_t) \|^2 \ge  y } \cdot \epsilon_t^{m+1} \| \nabla f(\theta_t) \|^2,
\end{align}
where \( \Delta_{f_t} := f(\theta_t) - f(\theta_{t+1}).\) 
 
\end{lem}

\section{Proofs of Lemmas}
\subsection{The Proof of Lemma \ref{loss_bound}}
\begin{proof}
	{For $\forall x\in \mathbb{R}^{N}$,	we define function
	\begin{equation}\nonumber
	\begin{aligned}
	g(t)=f\bigg(x+t\frac{x'-x}{\|x'-x\|}\bigg),
	\end{aligned}
	\end{equation}where $x'$ is a constant point such that   $x'-x$ is parallel to $\nabla f(x)$. By taking the derivative, we obtain
	\begin{equation}\label{qcxzd}
	\begin{aligned}
	g'(t)=\nabla_{x+t\frac{x'-x}{\|x'-x\|}}f\bigg(x+t\frac{x'-x}{\|x'-x\|}\bigg)^{T}\frac{x'-x}{\|x'-x\|}.
	\end{aligned}
	\end{equation}Through the Lipschitz condition of $\nabla f(x)$, we get $\forall t_{1}, \ t_{2}$
    \begin{equation}\nonumber
        \begin{aligned}
        &\big|g'(t_1) - g'(t_2)\big| \\
        &= \Bigg| \Bigg( \nabla_{x + t \frac{x' - x}{\|x' - x\|}} f \left( x + t_1 \frac{x' - x}{\|x' - x\|} \right) \\
        &\quad - \nabla_{x + t \frac{x' - x}{\|x' - x\|}} f \left( x + t_2 \frac{x' - x}{\|x' - x\|} \right) \Bigg)^T \frac{x' - x}{\|x' - x\|} \Bigg| \\
        &\le \Bigg\| \nabla_{x + t \frac{x' - x}{\|x' - x\|}} f \left( x + t_1 \frac{x' - x}{\|x' - x\|} \right) \\
        &\quad - \nabla_{x + t \frac{x' - x}{\|x' - x\|}} f \left( x + t_2 \frac{x' - x}{\|x' - x\|} \right) \Bigg\| \Bigg\| \frac{x' - x}{\|x' - x\|} \Bigg\| \\
        &\le \mathcal{L} |t_1 - t_2|.
        \end{aligned}
    \end{equation}

    So $g'(t)$ satisfies the Lipschitz  condition, and we have $\inf_{t\in \mathbb{R}}g(t)\geq\inf_{x\in\mathbb{R}^{N}}f(x)>-\infty$. Let $g^{*}=\inf{x\in_{\mathbb{R}}}g(x)$, then it holds that for $\forall \ t_{0}\in \ \mathbb{R},$
	\begin{equation}\label{qwcdfs}
	\begin{aligned}
	g(0)-g^{*}\geq g(0)-g(t_{0}).
	\end{aligned}
	\end{equation}By using the \emph{Newton-Leibniz's} formula, we get that
	\begin{equation}\nonumber
	\begin{aligned}
	g(0)-g(t_{0})=\int_{t_{0}}^{0}g'(\alpha)d\alpha=\int_{t_{0}}^{0}\big(g'(\alpha)-g'(0)\big)d\alpha+\int_{t_{0}}^{0}g'(0)d\alpha.
	\end{aligned}
	\end{equation}Through the Lipschitz condition of $g'$, we get that
	\begin{equation}\nonumber
	\begin{aligned}
	g(0)-g(t_{0})\geq\int_{t_{0}}^{0}-\mathcal{L}|\alpha-0|d\alpha+\int_{t_{0}}^{0}g'(0)d\alpha=\frac{1}{2\mathcal{L}}\big(g'(0)\big)^{2}.
	\end{aligned}
	\end{equation}
{	Then we take a special value of $t_{0}$. Let $t_{0}=-g'(0)/\mathcal{L}$, then we get}
	\begin{equation}\label{8unii}
	\begin{aligned}
	&g(0)-g(t_{0})\geq-\int_{t_{0}}^{0}\mathcal{L}|\alpha|d\alpha+\int_{t_{0}}^{0}g(0)dt=-\frac{\mathcal{L}}{2}(0-t_{0})^{2}+g'(0)(-t_{0})\\&=-\frac{1}{2\mathcal{L}}\big(g'(0)\big)^{2}+\frac{1}{\mathcal{L}}\big(g'(0)\big)^{2}=\frac{1}{2\mathcal{L}}\big(g'(0)\big)^{2}.
	\end{aligned}
	\end{equation}Substituting Eq. \ref{8unii} into Eq. \ref{qwcdfs}, we get   
	\begin{equation}\nonumber
	\begin{aligned}
	g(0)-g^{*}\geq\frac{1}{2\mathcal{L}}\big(g'(0)\big)^{2}.
	\end{aligned}
	\end{equation}Due to $g^{*}\geq f^{*}$ and $\big(g'(0)\big)^{2}=\|\nabla f(x)\|^{2}$, it follows that
	\begin{equation}\nonumber
	\begin{aligned}
	\big\|\nabla f(x)\big\|^{2}\le 2\mathcal{L}\big(f(x)-f^{*}\big).
	\end{aligned}
	\end{equation}}

\end{proof}

\subsection{The Proof of Lemma \ref{descent_lemma}}
\begin{proof}
We compute $f(\theta_{t+1})-f(\theta_{t})$. According to the $L$-smooth condition, we obtain the following estimate:
\begin{align*}
f(\theta_{t+1})-f(\theta_{t})&\le \nabla f(\theta_{t})^{\top}(\theta_{t+1}-\theta_{t})+\frac{L}{2}\|\theta_{t+1}-\theta_{t}\|^{2}\notag\\&=-\epsilon_{t}\nabla f(\theta_{t})^{\top}g_{t}+\frac{L\epsilon_{t}^{2}}{2}\|g_t\|^{2}\notag\\&=-\epsilon_{t}\|\nabla f(\theta_{t})\|^{2}+\underbrace{\epsilon_{t}\nabla f(\theta_{t})^{\top}(\nabla f(\theta_{t})-g_{t})}_{M_{t}}+\frac{L\epsilon_{t}^{2}}{2}\|g_t\|^{2}.
\end{align*}
We complete the proof.
\end{proof}
\subsection{The Proof of Lemma \ref{thm:as_convergence:-1.0.0}}\label{dsfadsfdsafdasdfasdfasdfaadas}
\begin{proof}
For any \( \upsilon > 0,\ t_0\ge 1 \), we define the following first hitting time:  
\[
\tau_{t_0,\upsilon} := \min \{ t\ge t_0 \mid \|\nabla f(\theta_{t})\| \leq \upsilon \}.
\]
We multiply the indicator function \( \mathbb{I}_{\tau_{t_0,\upsilon} > t} \) on both sides of the descent inequality in Lemma \ref{descent_lemma} to obtain:
\begin{align*}
    \mathbb{I}_{\tau_{t_0,\upsilon} > t}\left(f(\theta_{t+1})-f^{*}\right)-\mathbb{I}_{\tau_{t_0,\upsilon} > t}\left(f(\theta_{t})-f^*\right)&\le-\mathbb{I}_{\tau_{t_0,\upsilon} > t}\epsilon_{t}\|\nabla f(\theta_{t})\|^{2}\\&\quad+\mathbb{I}_{\tau_{t_0,\upsilon} > t}\frac{L\epsilon_{t}^{2}}{2}\|g_t\|^{2}+\mathbb{I}_{\tau_{t_0,\upsilon} > t}M_{t},
\end{align*}
Noting that \(\mathbb{I}_{\tau_{t_0,\upsilon} > t+1}<\mathbb{I}_{\tau_{t_0,\upsilon} > t}\), we can further obtain:
\begin{align*}
    \mathbb{I}_{\tau_{t_0,\upsilon} > t+1}\left(f(\theta_{t+1})-f^{*}\right)-\mathbb{I}_{\tau_{t_0,\upsilon} > t}\left(f(\theta_{t})-f^*\right)&\le-\mathbb{I}_{\tau_{t_0,\upsilon} > t}\epsilon_{t}\|\nabla f(\theta_{t})\|^{2}+\mathbb{I}_{\tau_{t_0,\upsilon} > t}\frac{L\epsilon_{t}^{2}}{2}\|g_t\|^{2}\\&\quad+\mathbb{I}_{\tau_{t_0,\upsilon} > t}M_{t}.
\end{align*}
Taking the expectation on both sides of the above inequality, we obtain:
\begin{align*}
    &\quad\Expect\left[\mathbb{I}_{\tau_{t_0,\upsilon} > t+1}\left(f(\theta_{t+1})-f^{*}\right)\right]-\Expect\left[\mathbb{I}_{\tau_{t_0,\upsilon} > t}\left(f(\theta_{t})-f^*\right)\right]\\&\le -\Expect\left[\mathbb{I}_{\tau_{t_0,\upsilon} > t}\epsilon_{t}\|\nabla f(\theta_{t})\|^{2}\right]+\Expect\left[\mathbb{I}_{\tau_{t_0,\upsilon} > t}\frac{L\epsilon_{t}^{2}}{2}\|g_t\|^{2}\right]+\Expect\left[\mathbb{I}_{\tau_{t_0,\upsilon} > t}M_{t}\right]\\&\mathop{\le}^{(\text{i})}-\Expect\left[\mathbb{I}_{\tau_{t_0,\upsilon} > t}\epsilon_{t}\|\nabla f(\theta_{t})\|^{2}\right]+\Expect\left[\mathbb{I}_{\tau_{t_0,\upsilon} > t}\frac{L\epsilon_{t}^{2}}{2}\Expect\left[\|g_t\|^{2}|\mathscr{F}_{t-1}\right]\right]+0\\&\mathop{\le}^{\mathop{(\text{ii})}}-\epsilon_{t}\left(1-\frac{L\epsilon_{t}}{2}\left(G+\frac{1}{\upsilon^{2}}\right)\right)\Expect\left[\mathbb{I}_{\tau_{t_0,\upsilon} > t}\|\nabla f(\theta_{t})\|^{2}\right].
\end{align*}
In the above derivation, at step \( (\text{i}) \), we need to note that \( [\tau_{t_0,\upsilon} > t] \in \mathscr{F}_{t-1} \)~\footnote{Generally, for a stopping time \( \tau \), we can only ensure that \( [\tau > t] \in \mathscr{F}_{t} \). However, due to the specific construction of the sequence indices in this paper, we have the property: \[
[\tau_{t_0,\upsilon} > t] \in \mathscr{F}_{t-1}.
\]}. Step \( (\text{ii}) \) follows from the fact that when the indicator function \( \mathbb{I}_{[\tau_{\upsilon} > t]} \) takes the value \( 1 \), it implies that \( \|\nabla f(\theta_{t})\| > \upsilon \). By applying the weak growth condition (Assumption \ref{assump:stochastic_gradient}$\sim$Item (b)), this step is completed. 

Since Setting 1 ensures that \( \lim_{t \to +\infty} \epsilon_{t} = 0 \), we know that there exists a sufficiently large \( t_{\upsilon} \) such that for any \( t \geq t_{\upsilon} \), the following inequality holds:  
\[
1 - \frac{L \epsilon_{t}}{2} \left( G + \frac{1}{\upsilon^{2}} \right) \geq \frac{1}{2}.
\]
This implies that for any \( t \geq t_\upsilon \), we have  
\begin{align*}
  \Expect\left[\mathbb{I}_{\tau_{t_0,\upsilon} > t+1}\left(f(\theta_{t+1})-f^{*}\right)\right]-\Expect\left[\mathbb{I}_{\tau_{t_0,\upsilon} > t}\left(f(\theta_{t})-f^*\right)\right]\le-\frac{\epsilon_{t}}{2}\Expect\left[\mathbb{I}_{\tau_{t_0,\upsilon} > t}\|\nabla f(\theta_{t})\|^{2}\right].
\end{align*}
Next, we sum the above inequality over the index \( t \) from \( t_\upsilon \) to \( +\infty \). This yields:  
\begin{align}\label{sgd____1}
  \Expect\left[\mathbb{I}_{\tau_{t_0,\upsilon} > t+1}\left(f(\theta_{t+1})-f^{*}\right)\right]-\Expect\left[\mathbb{I}_{\tau_{t_0,\upsilon} > t}\left(f(\theta_{t})-f^*\right)\right]&\le-\frac{\epsilon_{t}}{2}\Expect\left[\mathbb{I}_{\tau_{t_0,\upsilon} > t}\|\nabla f(\theta_{t})\|^{2}\right].
\end{align}
Next, we sum Eq. \ref{sgd____1} over the index \( t \) from \( \max\{t_0,t_\upsilon\} \) to \(T \). This gives:
\begin{align*}
    \sum_{t=\max\{t_0,t_\upsilon\}}^{T}\frac{\epsilon_{t}}{2}\Expect\left[\mathbb{I}_{\tau_{t_0,\upsilon} > t}\|\nabla f(\theta_{t})\|^{2}\right]<\Expect\left[\mathbb{I}_{\tau_{t_0,\upsilon} > \max\{t_0,t_\upsilon\}}\left(f(\theta_{\max\{t_0,t_\upsilon\}})-f^*\right)\right].
\end{align*}
Noting
\begin{align*}
\Expect\left[\mathbb{I}_{\tau_{t_0,\upsilon} > t}\|\nabla f(\theta_{t})\|^{2}\right]\ge \Expect\left[\mathbb{I}_{\tau_{t_0,\upsilon} > t}\inf_{t_0\le k\le t}\|\nabla f(\theta_{k})\|^{2}\right],
\end{align*}
we acquire:
\begin{align*}
    \sum_{t=\max\{t_0,t_\upsilon\}}^{T}\frac{\epsilon_{t}}{2}\Expect\left[\mathbb{I}_{\tau_{t_0,\upsilon} > t}\inf_{t_0\le k\le t}\|\nabla f(\theta_{k})\|^{2}\right]<\Expect\left[\mathbb{I}_{\tau_{t_0,\upsilon} > t_{\upsilon}}\left(f(\theta_{\max\{t_0,t_\upsilon\}})-f^*\right)\right].
\end{align*}
Since the sequence \( \{\mathbb{E}\left[\mathbb{I}_{\tau_{t_0,\upsilon} > t}\inf_{t_0\le k\le t}\|\nabla f(\theta_{k})\|^{2}\right]\}_{t \geq t_0} \) is clearly monotonically decreasing, we readily obtain:  
\[
 \frac{1}{2}\left(\sum_{t=\max\{t_0,t_\upsilon\}}^{T}\epsilon_{t}\right)\Expect\left[\mathbb{I}_{\tau_{t_0,\upsilon} > T}\inf_{t_0\le k\le T}\|\nabla f(\theta_{k})\|^{2}\right]<\Expect\left[\mathbb{I}_{\tau_{t_0,\upsilon} > t_{\upsilon}}\left(f(\theta_{\max\{t_0,t_\upsilon\}})-f^*\right)\right],
\]
which means
\begin{align*}
 \Expect\left[\mathbb{I}_{\tau_{t_0,\upsilon} > T}\inf_{t_0\le k\le T}\|\nabla f(\theta_{k})\|^{2}\right]<\frac{2}{\sum_{t=1}^{T} \epsilon_{t}} \Expect\left[\mathbb{I}_{\tau_{t_0,\upsilon} > t_{\upsilon}}\left(f(\theta_{\max\{t_0,t_\upsilon\}})-f^*\right)\right].
\end{align*}
Since Setting \ref{assump:learning_rate} requires that  
\[
\sum_{t=1}^{+\infty} \epsilon_{t} = +\infty,
\]  
we can immediately obtain  

\begin{align}
&\lim_{T \rightarrow +\infty} 
\mathbb{E} \left[ \mathbb{I}_{\tau_{t_0,\upsilon} > T} 
\inf_{t_0\le k\le T}\|\nabla f(\theta_{k})\|^{2} \right] \notag \\
&\quad \leq \mathbb{E} \left[ \mathbb{I}_{\tau_{t_0,\upsilon} > t_{\upsilon}} 
\left( f(\theta_{\max\{t_0,t_\upsilon\}}) - f^* \right) \right] 
\lim_{T \rightarrow +\infty} \left( \sum_{t=\max\{t_0,t_\upsilon\}}^{T} \epsilon_t \right)^{-1} 
= 0.
\end{align}
The last equality holds because \( \sum_{t=t_\upsilon}^{\infty} \epsilon_{t} = +\infty \) ensures that the inverse sum tends to zero as \( T \to +\infty.\) We know that
\[\mathbb{I}_{\tau_{t_0,\upsilon} > T} =\I_{\inf_{t_0\le k\le T}\|\nabla f(\theta_{k})\|^{2}>\upsilon},\] which means
\[
\lim_{T \rightarrow +\infty} \mathbb{E} \left[\I_{\inf_{t_0\le k\le T}\|\nabla f(\theta_{k})\|^{2}>\upsilon}\inf_{t_0\le k\le T}\|\nabla f(\theta_{k})\|^{2}\right]  = 0.
\]
On the other hand, we clearly have:
\[\limsup_{T \rightarrow +\infty} \mathbb{E} \left[\I_{\inf_{t_0\le k\le T}\|\nabla f(\theta_{k})\|^{2}\le \upsilon}\inf_{t_0\le k\le T}\|\nabla f(\theta_{k})\|^{2}\right]\le \upsilon.\]
Combining above two inequalities, we get that
\begin{align*}
 \limsup_{T \rightarrow +\infty}   \mathbb{E} \left[\inf_{t_0\le k\le T}\|\nabla f(\theta_{k})\|^{2}\right]&\le  \limsup_{T \rightarrow +\infty} \mathbb{E} \left[\I_{\inf_{t_0\le k\le T}\|\nabla f(\theta_{k})\|^{2}\le \upsilon}\inf_{t_0\le k\le T}\|\nabla f(\theta_{k})\|^{2}\right]\\&\quad+\lim_{T \rightarrow +\infty} \mathbb{E} \left[\I_{\inf_{t_0\le k\le T}\|\nabla f(\theta_{k})\|^{2}>\upsilon}\inf_{t_0\le k\le T}\|\nabla f(\theta_{k})\|^{2}\right]\\&\le \upsilon.
\end{align*}
Due to the arbitrariness of $\upsilon,$ we can immediately obtain
\begin{align*}
 \lim_{T \rightarrow +\infty}   \mathbb{E} \left[\inf_{t_0\le k\le T}\|\nabla f(\theta_{k})\|^{2}\right]=0.
\end{align*}
Since $\{\inf_{t_0\le k\le T}\|\nabla f(\theta_{k})\|^{2}\}_{T\ge t_0}$ is a monotonically decreasing sequence, according to the \emph{Lebesgue's Monotone Convergence Theorem}, we know that the limit of the expectation can be exchanged. Specifically, we obtain the following expression:
\begin{align*}
\mathbb{E} \left[  \inf_{ k\ge t_0 }\|\nabla f(\theta_{k})\|^{2}\right]=\mathbb{E} \left[ \lim_{T \rightarrow +\infty}   \inf_{t_0\le k\le T}\|\nabla f(\theta_{k})\|^{2}\right]=\lim_{T \rightarrow +\infty}   \mathbb{E} \left[\inf_{t_0\le k\le T}\|\nabla f(\theta_{k})\|^{2}\right]=0.
\end{align*}
Since we obviously have $ \inf_{ k\ge t_0 }\|\nabla f(\theta_{k})\|^{2}\ge 0,$ the expectation being equal to $0$ implies that  $\inf_{ k\ge t_0 }\|\nabla f(\theta_{k})\|^{2} = 0\ \ \text{a.s.}$ Finally, due to the arbitrariness of $t_0,$ we have proven this lemma.
\end{proof}

\subsection{The Proof of Lemma \ref{lem:descent_prime}}
\begin{proof}
We first compute $f(\theta_{t+1}) - f(\theta_{t})$. According to the $L$-smooth condition, we obtain the following estimate:
\begin{align}\label{sgd_0}
    f(\theta_{t+1})-f(\theta_{t})&\le \nabla f(\theta_{t})^{\top}(\theta_{t+1}-\theta_{t})+\frac{L}{2}\|\theta_{t+1}-\theta_{t}\|^{2}\notag\\&=-\epsilon_{t}\nabla f(\theta_{t})^{\top}g_{t}+\frac{L\epsilon_{t}^{2}}{2}\|g_t\|^{2}\notag\\&=-\epsilon_{t}\|\nabla f(\theta_{t})\|^{2}+\underbrace{\epsilon_{t}\nabla f(\theta_{t})^{\top}(\nabla f(\theta_{t})-g_{t})}_{M_{t,1}}+\frac{L\epsilon_{t}^{2}}{2}\|g_t\|^{2}.
 \end{align}
 For any $y > 0$, we multiply both sides of Eq. \ref{sgd_0} by the indicator function $\I_{\|\nabla f(\theta_{t})\|^{2} \ge y}$. This represents considering the derived properties of Eq. \ref{sgd_0} when the event $\left[\|\nabla f(\theta_{t})\|^{2} \ge y\right]$ occurs. Specifically, we have:
 \begin{align}\label{Delta}
    \I_{\|\nabla f(\theta_{t})\|^{2} \ge y}\big( \underbrace{f(\theta_{t+1})-f(\theta_{t})}_{-\Delta_{f_t}}\big)&\le-\I_{\|\nabla f(\theta_{t})\|^{2} \ge y}\cdot\epsilon_{t}\|\nabla f(\theta_{t})\|^{2}+\I_{\|\nabla f(\theta_{t})\|^{2} \ge y}M_{t}\notag\\&+\I_{\|\nabla f(\theta_{t})\|^{2} \ge y}\frac{L\epsilon_{t}^{2}}{2}\|g_{t}\|^{2}.
 \end{align}
 After simplification, we obtain:
 \begin{align}\label{sgd_1}
    \I_{\|\nabla f(\theta_{t})\|^{2} \ge y}\cdot\epsilon_{t}\|\nabla f(\theta_{t})\|^{2}&\le \I_{\|\nabla f(\theta_{t})\|^{2} \ge y}\Delta_{f_t}+\I_{\|\nabla f(\theta_{t})\|^{2} \ge y}M_{t}\notag\\&+\I_{\|\nabla f(\theta_{t})\|^{2} \ge y}\frac{L\epsilon_{t}^{2}}{2}\|g_{t}\|^{2}.
 \end{align}
 For any $n \ge 1$, we multiply both sides of the above inequality by $\epsilon_{t}^{m-1}$, and we obtain:
 \begin{align}\label{sgd_01}
    \I_{\|\nabla f(\theta_{t})\|^{2} \ge y}\cdot\epsilon^{m}_{t}\|\nabla f(\theta_{t})\|^{2}&\le \I_{\|\nabla f(\theta_{t})\|^{2} \ge y}\epsilon_{t}^{m-1}\Delta_{f_t}+\I_{\|\nabla f(\theta_{t})\|^{2} \ge y}\epsilon_{t}^{m-1}M_{t}\notag\\&+\I_{\|\nabla f(\theta_{t})\|^{2} \ge y}\frac{L\epsilon_{t}^{m+1}}{2}\|g_{t}\|^{2}.
 \end{align}
 Neyt, using the \emph{weak growth condition} (Item (b) from Assumption \ref{assump:stochastic_gradient}), we perform the following estimation for the quadratic error term $\I_{\|\nabla f(\theta_{t})\|^{2} \ge y}\frac{L\epsilon_{t}^{m+1}}{2}\|g_{t}\|^{2}$. We have:
 \begin{align*}
 \I_{\|\nabla f(\theta_{t})\|^{2} \ge y}\|g_{t}\|^{2}&\le \I_{\|\nabla f(\theta_{t})\|^{2} \ge y}G(\|\nabla f(\theta_{t})\|^{2}+1)\\&=\I_{\|\nabla f(\theta_{t})\|^{2} \ge y}G\bigg(1+\frac{1}{\|\nabla f(\theta_{t})\|^{2}}\bigg)\|\nabla f(\theta_{t})\|^{2}\\&<\I_{\|\nabla f(\theta_{t})\|^{2} \ge y}G\Big(1+\frac{1}{y}\Big)\|\nabla f(\theta_{t})\|^{2}.
 \end{align*}
 Substituting the above estimate back into Eq. \ref{sgd_1}, we obtain:
 \begin{align}\label{sgd_2}
    \I_{\|\nabla f(\theta_{t})\|^{2} \ge y}\cdot\epsilon^{m}_{t}\|\nabla f(\theta_{t})\|^{2}&\le \I_{\|\nabla f(\theta_{t})\|^{2} \ge y}\epsilon_{t}^{m-1}\Delta_{f_t}+\I_{\|\nabla f(\theta_{t})\|^{2} \ge y}\epsilon_{t}^{m-1}M_{t}\notag\\&+\frac{LG}{2}\Big(1+\frac{1}{y}\Big)\I_{\|\nabla f(\theta_{t})\|^{2} \ge y}\cdot\epsilon_{t}^{m+1}\|\nabla f(\theta_{t})\|^{2}.
 \end{align}
 With this, we complete the proof.

 \end{proof}


\subsection{The Proof of Lemma \ref{lem_1123}}\label{lem+1123}
\begin{proof}
We have
\begin{align*}
\I_{f(\theta_{t}) - f^* < D_{\eta}} \left( f(\theta_{t+1}) - f(\theta_{t}) \right)&=\I_{f(\theta_{t}) - f^* < D_{\eta}}\nabla f(\theta_{\xi_{t}})^{\top}(\theta_{t+1}-\theta_{t})\\&=\I_{f(\theta_{t}) - f^* < D_{\eta}}\nabla f(\theta_{{t}})^{\top}(\theta_{t+1}-\theta_{t})\\&+\I_{f(\theta_{t}) - f^* < D_{\eta}}(\nabla f(\theta_{\xi_{t}})-\nabla f(\theta_{{t}}))^{\top}(\theta_{t+1}-\theta_{t}).
\end{align*}
Taking absolute values on both sides, we obtain:
\begin{align}\label{sgd_i2}
\I_{f(\theta_{t}) - f^* < D_{\eta}} \left| f(\theta_{t+1}) - f(\theta_{t}) \right| &\leq \I_{f(\theta_{t}) - f^* < D_{\eta}} \|\nabla f(\theta_{t})\|\cdot\|\theta_{t+1}-\theta_{t}\|\notag\\&+\I_{f(\theta_{t}) - f^* < D_{\eta}}\|\nabla f(\theta_{\xi_{t}})-\nabla f(\theta_{{t}})\|\cdot\|\theta_{t+1}-\theta_{t}\|\notag\\&\mathop{\le}^{\text{L-Smooth}}\I_{f(\theta_{t}) - f^* < D_{\eta}} \epsilon_{t}\|\nabla f(\theta_{t})\|\|g_{t}\|+\I_{f(\theta_{t}) - f^* < D_{\eta}} L\epsilon_{t}^{2}\|g_{t}\|^{2}\notag\\&\mathop{\le}^{\text{Lemma \ref{loss_bound}}}\I_{f(\theta_{t}) - f^* < D_{\eta}} \epsilon_{t}\sqrt{2LD_{\eta}}\|g_{t}\|+\I_{f(\theta_{t}) - f^* < D_{\eta}} L\epsilon_{t}^{2}\|g_{t}\|^{2}\notag\\&\mathop{\le}^{\text{\emph{AM-GM} inequality}} \frac{\nu}{2}+\left(1+\frac{D_{\eta}}{\nu}\right)L\I_{f(\theta_{t}) - f^* < D_{\eta}} \epsilon_{t}^{2}\|g_{t}\|^{2}.
\end{align} 
Next, we apply \emph{Young's} inequality to continue the expansion, which gives us:
\begin{align*}
 \I_{f(\theta_{t}) - f^* < D_{\eta}} \left| f(\theta_{t+1}) - f(\theta_{t}) \right| &\leq \frac{3}{4}\nu+\overline{C}_{\nu}\I_{f(\theta_{t}) - f^* < D_{\eta}} \epsilon_{t}^{p}\|g_{t}\|^{p}\\&<\nu+\overline{C}_{\nu}\I_{f(\theta_{t}) - f^* < D_{\eta}} \epsilon_{t}^{p}\|g_{t}\|^{p},
\end{align*}
where\[\overline{C}_{\nu}:=\left(1+\frac{D_{\eta}}{\nu}\right)L\cdot\frac{2}{p}\left(\frac{4\left(1+\frac{D_{\eta}}{\nu}\right)L\left(1-\frac{2}{p}\right)}{\nu}\right)^{\frac{p-2}{2}},\]
that is 
\begin{align}\label{fdasfdas}
\left(\I_{f(\theta_{t}) - f^* < D_{\eta}} \left| f(\theta_{t+1}) - f(\theta_{t}) \right|-\nu\right)_{+} &\leq \overline{C}_{\nu}\I_{f(\theta_{t}) - f^* < D_{\eta}} \epsilon_{t}^{p}\|g_{t}\|^{p}.
\end{align}
Based on Setting \ref{assump:learning_rate}, we can conclude that
$$\sum_{t=1}^{+\infty}\Expect\left[\left(\I_{f(\theta_{t}) - f^* < D_{\eta}} \left| f(\theta_{t+1}) - f(\theta_{t}) \right|-\nu\right)_{+}\right]\le \overline{C}_{\nu}M_{p} \sum_{t=1}^{+\infty}\epsilon_{t}^{p}:=C_{\nu}.$$
With this, we complete the proof.

\end{proof}

\subsection{The Proof of Lemma \ref{iteration}}\label{P_iteration}
\begin{proof}
We define the following sequence of stopping times \( \{\tau_n\}_{n \ge 1} \):
\begin{align*}
&\tau_{1}:=\min\{t\ge 1:f(\theta_{t})-f^*\ge a\},\ \tau_{2}:=\min\{t\ge \tau_{1}:f(\theta_{t})-f^*\ge b\ \ \text{or}\ \ f(\theta_{t})-f^*< a\},\\& \tau_{3}:=\min\{t\ge \tau_{2}:f(\theta_{t})-f^*\ge  c\ \ \text{or}\ \ f(\theta_{t})-f^*< a\},\\& \tau_{4}:=\min\{t\ge \tau_{3}:f(\theta_{t})-f^*< a\}...,\\& \tau_{4i-3}:=\min\{t\ge \tau_{4i-4}:f(\theta_{t})-f^*\ge a\},\\& \tau_{4i-2}:=\min\{t\ge \tau_{4i-3}:f(\theta_{t})-f^*\ge b\ \ \text{or}\ \ f(\theta_{t})-f^*< a\},\\& \tau_{4i-1}:=\min\{t\ge \tau_{4i-2}:f(\theta_{t})-f^*\ge c\ \ \text{or}\ \ f(\theta_{t})-f^*< a\}\\& \tau_{4i}:=\min\{t\ge \tau_{4i-1}:f(\theta_{t})-f^*< a\}.
\end{align*}

Next, for any $T \geq 1$, we define the truncated stopping time as $\tau_{n,T} := \tau_{n} \wedge T$. Then, by applying Lemma \ref{lem:descent_prime} with $y = \delta^{2}_{a,b}$ on the interval $[\tau_{4i-2,T}, \tau_{4i,T})$ when $\tau_{4i-2,T}<\tau_{4i-1,T},$ we obtain that $\forall\ t\in[\tau_{4i-2,T}, \tau_{4i-1,T})$ when $\tau_{4i-2,T}<\tau_{4i-1,T},$ there is:
\begin{align*}
\I^{(i)}\I_{\|\nabla f(\theta_{t})\|^{2} \ge \delta^{2}_{a,b}}\cdot\epsilon^{m}_{t}\|\nabla f(\theta_{t})\|^{2}&\le \I^{(i)}\I_{\|\nabla f(\theta_{t})\|^{2} \ge \delta^{2}_{a,b}}\epsilon_{t}^{m-1}\Delta_{f_t}+\I^{(i)}\I_{\|\nabla f(\theta_{t})\|^{2} \ge \delta^{2}_{a,b}}\epsilon_{t}^{m-1}M_{t}\notag\\&+\I^{(i)}\I_{\|\nabla f(\theta_{t})\|^{2} \ge \delta^{2}_{a,b}}\frac{L\epsilon_{t}^{m+1}}{2}\|g_{t}\|^{2},
\end{align*}
where $\I^{(i)}:=\I_{[\tau_{4i-2,T}<\tau_{4i,T}]}.$ Summing the indices $t$ in the above inequality from $\tau_{4i-2,T}$ to $\tau_{4i-1,T}-1$ under the event $[\tau_{4i-2,T}<\tau_{4i-1,T}],$ we obtain:
\begin{align*}
\I^{(i)}\sum_{t=\tau_{4i-2,T}}^{\tau_{4i-1,T}-1}\I_{\|\nabla f(\theta_{t})\|^{2} \ge \delta^{2}_{a,b}}\cdot\epsilon^{m}_{t}\|\nabla f(\theta_{t})\|^{2}&\le \I^{(i)}\sum_{t=\tau_{4i-2,T}}^{\tau_{4i-1,T}-1}\I_{\|\nabla f(\theta_{t})\|^{2} \ge \delta^{2}_{a,b}}\epsilon_{t}^{m-1}\Delta_{f_t}\\&+\I^{(i)}\sum_{t=\tau_{4i-2,T}}^{\tau_{4i-1,T}-1}\I_{\|\nabla f(\theta_{t})\|^{2} \ge \delta^{2}_{a,b}}\epsilon_{t}^{m-1}M_{t}\notag\\&+\frac{L}{2}\I^{(i)}\sum_{t=\tau_{4i-2,T}}^{\tau_{4i-1,T}-1}\I_{\|\nabla f(\theta_{t})\|^{2} \ge \delta^{2}_{a,b}}\cdot\epsilon_{t}^{m+1}\|g_t\|^{2}.
\end{align*}
It is easy to see that the event $[\tau_{4i-2,T} < \tau_{4i-1,T}]$ is equivalent to $[\tau_{4i-1} < T] \cap [\tau_{4i-1} < \tau_{4i}]$. Therefore, when $\I_{[\tau_{4i-2,T} < \tau_{4i-1,T}]} = 1$, we have $\I_{\|\nabla f(\theta_{t})\|^2 \ge \delta^{2}_{a,b}} = 1$ for all $t \in [\tau_{4i-2,T}, \tau_{4i-1,T})$. As a result, we can remove the indicator function $\I_{\|\nabla f(\theta_{t})\|^2 \ge \delta^{2}_{a,b}}$ from both sides of the above inequality, yielding:
\begin{align}\label{sgd_-100}
\I^{(i)}\sum_{t=\tau_{4i-2,T}}^{\tau_{4i-1,T}-1}\epsilon^{m}_{t}\|\nabla f(\theta_{t})\|^{2}&\le \I^{(i)}\sum_{t=\tau_{4i-2,T}}^{\tau_{4i-1,T}-1}\epsilon_{t}^{m-1}\Delta_{f_t}+\I^{(i)}\sum_{t=\tau_{4i-2,T}}^{\tau_{4i-1,T}-1}\epsilon_{t}^{m-1}M_{t}\notag\\&+\frac{L}{2}\I^{(i)}\sum_{t=\tau_{4i-2,T}}^{\tau_{4i-1,T}-1}\epsilon_{t}^{m+1}\|g_{t}\|^{2}\frac{L}{2}\I^{(i)}\sum_{t=\tau_{4i-2,T}}^{\tau_{4i-1,T}-1}\epsilon_{t}^{m+1}\|g_{t}\|^{2}\notag\\&\mathop{\le}^{(\text{i})}\I^{(i)}\epsilon^{m-1}_{\tau_{4i-2,T}}(f(\theta_{\tau_{4i-2,T}})-f^*)+\I^{(i)}\sum_{t=\tau_{4i-2,T}}^{\tau_{4i-1,T}-1}\epsilon_{t}^{m-1}M_{t}\notag\\&+\frac{L}{2}\I^{(i)}\sum_{t=\tau_{4i-2,T}}^{\tau_{4i-1,T}-1}\epsilon_{t}^{m+1}\|g_{t}\|^{2}.
\end{align}
In step $(\text{i}),$ the transformation is mainly applied to the first term on the left-hand side of the corresponding inequality. We have:
\begin{align*}
\I^{(i)}\sum_{t=\tau_{4i-2,T}}^{\tau_{4i-1,T}-1}\epsilon_{t}^{m-1}\Delta_{f_t}&\le \I^{(i)}\sum_{t=\tau_{4i-2,T}}^{\tau_{4i-1,T}-1}\left(\epsilon_{t}^{m-1}(f(\theta_{t})-f^*)-\epsilon_{t+1}^{m-1}(f(\theta_{t+1})-f^*)\right)\\&<\I^{(i)}\epsilon^{m-1}_{\tau_{4i-2,T}}(f(\theta_{\tau_{4i-2,T}})-f^*).
\end{align*}
Taking the expectation on both sides of Eq. \ref{sgd_-100}, we obtain:
\begin{align}\label{sgd_-1000}
\Expect\Bigg[\I^{(i)}\sum_{t=\tau_{4i-2,T}}^{\tau_{4i-1,T}-1}\epsilon^{m}_{t}\|\nabla f(\theta_{t})\|^{2}\Bigg]&{\le}\underbrace{\Expect\left[\I^{(i)}\epsilon^{m-1}_{\tau_{4i-2,T}}(f(\theta_{\tau_{4i-2,T}})-f^*)\right]}_{\Theta_{i,T,1}}\notag\\&+\underbrace{\Expect\left[\I^{(i)}\sum_{t=\tau_{4i-2,T}}^{\tau_{4i-1,T}-1}\epsilon_{t}^{m-1}M_{t}\right]}_{\Theta_{i,T,2}}\notag\\&+\frac{L}{2}\left[\I^{(i)}\sum_{t=\tau_{4i-2,T}}^{\tau_{4i-1,T}-1}\epsilon_{t}^{m+1}\|g_{t}\|^{2}\right].
\end{align}
First, let's handle $\Theta_{t,T,1}$. According to the definition of stopping times, when $\I^{(i)} = 1$, that is, when $\tau_{4i-2,T} < \tau_{4i,T}\le T,$ we have $f(\theta_{\tau_{4i-2,T}-1})-f^*=f(\theta_{\tau_{4i-2}-1})-f^*\le b,$ that means:
\begin{align}\label{sgd_-400}
   \Theta_{i,T,1}&\le \Expect\left[\I^{(i)}\epsilon^{m-1}_{\tau_{4i-2,T}}(f(\theta_{\tau_{4i-2,T}-1})-f^*)\right]\notag\\&\quad+\Expect\left[\I^{(i)}\epsilon^{m-1}_{\tau_{4i-2,T}}(f(\theta_{\tau_{4i-2,T}})-f(\theta_{\tau_{4i-2,T}-1}))\right]\notag\\&\le b\Expect\left[\I^{(i)}\epsilon_{\tau_{4i-2,T}}^{m-1}\right]+\Expect\left[\I^{(i)}\epsilon^{m-1}_{\tau_{4i-2,T}}(f(\theta_{\tau_{4i-2,T}})-f(\theta_{\tau_{4i-2,T}-1}))\right]\notag\\&=\frac{3b}{2}\Expect\left[\I^{(i)}\epsilon_{\tau_{4i-2,T}}^{m-1}\right]-\frac{b}{2}\Expect\left[\I^{(i)}\epsilon_{\tau_{4i-2,T}}^{m-1}\right]\notag\\&\quad+\Expect\left[\I^{(i)}\epsilon^{m-1}_{\tau_{4i-2,T}}(f(\theta_{\tau_{4i-2,T}})-f(\theta_{\tau_{4i-2,T}-1}))\right]\notag\\&\le\frac{3b}{2}\Expect\left[\I^{(i)}\epsilon_{\tau_{4i-2,T}}^{m-1}\right]+ \Expect\left[\I^{(i)}\epsilon^{m-1}_{\tau_{4i-2,T}}\overline{\Delta}_{\tau_{4i-2,T}-1,b/2}\right].
\end{align}
Then we aim to address $\Theta_{i,T,2}$. Upon observation, it is evident that for any $n, k$, the stopping time $\tau_{k}$ satisfies the following additional property: $[\tau_{k} = n] \in \mathscr{F}_{n-1}$. This implies that the preceding time, $\tau_{k} - 1$, is also a stopping time. Therefore, for $\Theta_{i,T,2}.$ 
\begin{align*}
\Theta_{i,T,2}&=\Expect\left[\Expect\left[\I^{(i)}\sum_{t=\tau_{4i-2,T}}^{\tau_{4i-1,T}-1}\epsilon_{t}^{m-1}M_{t}\Bigg|\mathscr{F}_{\tau_{4i-2,T}-1}\right]\right] \\&=\Expect\left[\I^{(i)}\Expect\left[\sum_{t=\tau_{4i-2,T}}^{\tau_{4i-1,T}-1}\epsilon_{t}^{m-1}M_{t}\Bigg|\mathscr{F}_{\tau_{4i-2,T}-1}\right]\right]\\& \mathop{=}^{\text{\emph{Doob's Stopped} theorem}}\Expect\left[\I^{(i)}\Expect\left[\sum_{t=\tau_{4i-2,T}}^{\tau_{4i-1,T}-1}\epsilon_{t}^{m-1}\Expect[M_{t}|\mathscr{F}_{t-1}]\Bigg|\mathscr{F}_{\tau_{4i-2,T}-1}\right]\right]\\&=0
\end{align*}
We combine the estimates related to $\Theta_{i,T,2}$ from the above expression with those related to $\Theta_{i,T,1}$ in Eq. \ref{sgd_-400}, and substitute both back into Eq. \ref{sgd_-1000}, we obtain:
\begin{align}\label{sgd_-101}
\Expect\Bigg[\I^{(i)}\sum_{t=\tau_{4i-2,T}}^{\tau_{4i-1,T}-1}\epsilon^{m}_{t}\|\nabla f(\theta_{t})\|^{2}\Bigg]&{\le}\frac{3b}{2}\Expect\left[\I^{(i)}\epsilon_{\tau_{4i-2,T}}^{m-1}\right]+ \Expect\left[\I^{(i)}\epsilon^{m-1}_{\tau_{4i-2,T}}\overline{\Delta}_{\tau_{4i-2,T}-1,b/2}\right]\notag\\&\quad+\frac{L}{2}\left[\I^{(i)}\sum_{t=\tau_{4i-2,T}}^{\tau_{4i-1,T}-1}\epsilon_{t}^{m+1}\|g_{t}\|^{2}\right]\notag\\&\le \frac{3b}{2}\Expect\left[\I^{(i)}\epsilon_{\tau_{4i-2,T}}^{m-1}\right]+ \Expect\left[\I^{(i)}\epsilon^{m-1}_{\tau_{4i-2,T}}\overline{\Delta}_{\tau_{4i-2,T}-1,b/2}\right]\notag\\&\quad+\frac{L}{2}\Expect\left[\I^{(i)}\sum_{t=\tau_{4i-2,T}}^{\tau_{4i-1,T}-1}\epsilon_{t}^{m+1}\|g_{t}\|^{2}\right].
\end{align}
Next, we will address the first term in the above inequality. 
Here, we examine the properties of \(\I^{(i)}\). When \(\I^{(i)} = 1\), we have the event \([\tau_{4i-2,T} < \tau_{4i,T}]\). This implies that both \([\tau_{4i-2} < \tau_{4i}]\) and \([\tau_{4i} < T]\) hold simultaneously. Consequently, we obtain:
\[\I^{(i)}\big(f(\theta_{\tau_{4i-2,T}}) - f^* \big)>\I^{(i)} b.\]
This implies that
\begin{align}\label{st}\I^{(i)}\big(f(\theta_{\tau_{4i-2,T}})-f(\theta_{\tau_{4i-3,T}-1})\big)>\I^{(i)}(b-a).\end{align}
That is to say,
\[\I^{(i)}\big(\epsilon_{\tau_{4i-2,T}}^{\frac{m-1}{2}}(f(\theta_{\tau_{4i-2,T}})-f^*)-\epsilon_{\tau_{4i-2,T}}^{\frac{m-1}{2}}(f(\theta_{\tau_{4i-3,T}-1})-f^*)\big)>\I^{(i)}\epsilon_{\tau_{4i-2,T}}^{m-1}(b-a).\]
Since $f(\theta_{\tau_{4i-3,T}-1}) - f^* \le a$, we easily obtain:
\begin{align*}
&\I^{(i)}\big(\epsilon_{\tau_{4i-2,T}}^{\frac{m-1}{2}}(f(\theta_{\tau_{4i-2,T}})-f^*)-\epsilon_{\tau_{4i-3,T}-1}^{\frac{m-1}{2}}(f(\theta_{\tau_{4i-3,T}-1})-f^*)+a(\epsilon_{\tau_{4i-3,T}-1}^{\frac{m-1}{2}}-\epsilon_{\tau_{4i-2,T}}^{\frac{m-1}{2}})\big)\\&\ge\I^{(i)}\big(\epsilon_{\tau_{4i-2,T}}^{\frac{m-1}{2}}(f(\theta_{\tau_{4i-2,T}})-f^*)-\epsilon_{\tau_{4i-2,T}}^{\frac{m-1}{2}}(f(\theta_{\tau_{4i-3,T}-1})-f^*)\big) \\&>\I^{(i)}\epsilon_{\tau_{4i-2,T}}^{\frac{m-1}{2}}(b-a).
\end{align*}
We use the \emph{Descent Lemma} (Lemma~\ref{descent_lemma}) to bound
\[
\epsilon_{\tau_{4i-2,T}}^{\frac{m-1}{2}} (f(\theta_{\tau_{4i-2,T}}) - f^*)
-
\epsilon_{\tau_{4i-3,T}-1}^{\frac{m-1}{2}} (f(\theta_{\tau_{4i-3,T}-1}) - f^*),
\]
yielding:
\begin{align*}
    \I^{(i)}\epsilon_{\tau_{4i-2,T}}^{\frac{m-1}{2}}(b-a)&<\I^{(i)}\big(\epsilon_{\tau_{4i-2,T}}^{\frac{m-1}{2}}(f(\theta_{\tau_{4i-2,T}})-f^*)-\epsilon_{\tau_{4i-3,T}-1}^{\frac{m-1}{2}}(f(\theta_{\tau_{4i-3,T}-1})-f^*)\notag\\&+a(\epsilon_{\tau_{4i-3,T}-1}^{\frac{m-1}{2}}-\epsilon_{\tau_{4i-2,T}}^{\frac{m-1}{2}})\big)\notag\\&< \I^{(i)}\sum_{t=\tau_{4i-3,T}}^{\tau_{4i-2,T}-1}\epsilon_{t}^{\frac{m-1}{2}}\big(f(\theta_{t+1})-f(\theta_{t})\big)+a(\epsilon_{\tau_{4i-3,T}-1}^{\frac{m-1}{2}}-\epsilon_{\tau_{4i-2,T}}^{\frac{m-1}{2}})\\&+ \I^{(i)}\epsilon_{\tau_{4i-3,T}-1}^{\frac{m-1}{2}}\overline{\Delta}_{\tau_{4i-3,T}-1,(b-a)/2}+\I^{(i)}\epsilon_{\tau_{4i-2,T}}^{\frac{m-1}{2}}\frac{b-a}{2}.
\end{align*}
We multiply both sides of the above inequality by $\epsilon_{\tau_{4i-2,T}}^{\frac{m-1}{2}},$ and noting $\epsilon_{\tau_{4i-2,T}}^{\frac{m-1}{2}}<\epsilon_{1}^{\frac{m-1}{2}},$ we get:
 \begin{align}\label{adam_-201}
    \frac{\I^{(i)}\epsilon_{\tau_{4i-2,T}}^{m-1}(b-a)}{2}&<\I^{(i)}\epsilon_{\tau_{4i-2,T}}^{\frac{m-1}{2}}\sum_{t=\tau_{4i-3,T}}^{\tau_{4i-2,T}-1}\epsilon_{t}^{\frac{m-1}{2}}\big(f(\theta_{t+1})-f(\theta_{t})\big)+\I^{(i)}\epsilon_{1}^{\frac{m-1}{2}}\overline{\Delta}_{\tau_{4i-3,T}-1,(b-a)/2}\notag\\&+\epsilon_{1}^{\frac{m-1}{2}}\I^{(i)}a(\epsilon_{\tau_{4i-3,T}-1}^{\frac{m-1}{2}}-\epsilon_{\tau_{4i-2,T}}^{\frac{m-1}{2}})\notag\\&\mathop{<}^{\text{Lemma \ref{descent_lemma}}} \I^{(i)}\frac{L}{2}\sum_{t=\tau_{4i-3,T}}^{\tau_{4i-2,T}-1}\epsilon_{t}^{m+1}\|g_{t}\|^{2}+ \I^{(i)}\epsilon_{\tau_{4i-2,T}}^{\frac{m-1}{2}}\sum_{t=\tau_{4i-3,T}-1}^{\tau_{4i-2,T}-1}\epsilon_{t}^{\frac{m-1}{2}}M_{t}\notag\\&+\epsilon_{1}^{\frac{m-1}{2}}\I^{(i)}a(\epsilon_{\tau_{4i-3,T}-1}^{\frac{m-1}{2}}-\epsilon_{\tau_{4i-2,T}}^{\frac{m-1}{2}})+\I^{(i)}\epsilon_{1}^{\frac{m-1}{2}}\overline{\Delta}_{\tau_{4i-3,T}-1,(b-a)/2}\notag\\&\mathop{\le}^{\text{\emph{AM-GM} inequality}}  \I^{(i)}\frac{L}{2}\sum_{t=\tau_{4i-3,T}}^{\tau_{4i-2,T}-1}\epsilon_{t}^{m+1}\|g_{t}\|^{2}+ \I^{(i)}\frac{(b-a)\epsilon_{\tau_{4i-2,T}}^{m-1}}{4}\notag\\&+ \I^{(i)}\frac{1}{b-a}\left(\sum_{t=\tau_{4i-3,T}}^{\tau_{4i-2,T}-1}\epsilon_{t}^{\frac{m-1}{2}}M_{t}\right)^{2}\notag\\&+\I^{(i)}a(\epsilon_{\tau_{4i-3,T}-1}^{\frac{m-1}{2}}-\epsilon_{\tau_{4i-2,T}}^{\frac{m-1}{2}})+\I^{(i)}\epsilon_{1}^{\frac{m-1}{2}}\overline{\Delta}_{\tau_{4i-3,T}-1,(b-a)/2}.
\end{align}
This implies that the following equation holds:
 \begin{align*}
    \frac{\I^{(i)}\epsilon_{\tau_{4i-2,T}}^{m-1}(b-a)}{4}&<\I^{(i)}\frac{L}{2}\sum_{t=\tau_{4i-3,T}}^{\tau_{4i-2,T}-1}\epsilon_{t}^{m+1}\|g_{t}\|^{2}+ \I^{(i)}\frac{1}{b-a}\left(\sum_{t=\tau_{4i-3,T}}^{\tau_{4i-2,T}-1}\epsilon_{t}^{\frac{m-1}{2}}M_{t}\right)^{2}\notag\\&+\I^{(i)}a(\epsilon_{\tau_{4i-3,T}-1}^{\frac{m-1}{2}}-\epsilon_{\tau_{4i-2,T}}^{\frac{m-1}{2}})+\I^{(i)}\epsilon_{1}^{\frac{m-1}{2}}\overline{\Delta}_{\tau_{4i-3,T}-1,(b-a)/2}.
\end{align*}
Taking the expectation on both sides of the above inequality and scaling the indicator function $\I^{(i)}$ on the right side to $1$, we obtain:
\begin{align}\label{sgd_-500}
\frac{b-a}{4}\Expect\left[{\I^{(i)}\epsilon_{\tau_{4i-2,T}}^{m-1}}\right]&<\frac{L}{2}\Expect\left[\sum_{t=\tau_{4i-3,T}}^{\tau_{4i-2,T}-1}\epsilon_{t}^{m+1}\|g_{t}\|^{2}\right]+ \frac{1}{b-a}\Expect\left[\sum_{t=\tau_{4i-3,T}-1}^{\tau_{4i-2,T}-1}\epsilon_{t}^{\frac{m-1}{2}}M_{t}\right]^{2}\notag\\&+a\Expect[\epsilon_{\tau_{4i-3,T}-1}^{\frac{m-1}{2}}-\epsilon_{\tau_{4i-2,T}}^{\frac{m-1}{2}}]+\epsilon_{1}^{\frac{m-1}{2}}\Expect\left[\overline{\Delta}_{\tau_{4i-3,T}-1,(b-a)/2}\right]\notag\\&{\le}\frac{L}{2}\Expect\left[\sum_{t=\tau_{4i-3,T}}^{\tau_{4i-2,T}-1}\epsilon_{t}^{m+1}\|g_{t}\|^{2}\right]\notag\\&+ \frac{1}{b-a}\Expect\left[\sum_{t=1}^{\tau_{4i-2,T}-1}\epsilon_{t}^{\frac{m-1}{2}}M_{t}-\sum_{t=1}^{\tau_{4i-3,T}-1}\epsilon_{t}^{\frac{m-1}{2}}M_{t}\right]^{2}\notag\\&+a\Expect[\epsilon_{\tau_{4i-3,T}}^{\frac{m-1}{2}}-\epsilon_{\tau_{4i-2,T}}^{\frac{m-1}{2}}]+\epsilon_{1}^{\frac{m-1}{2}}\Expect\left[\overline{\Delta}_{\tau_{4i-3,T}-1,(b-a)/2}\right]\notag\\&\mathop{\le}^{(\text{i})}\frac{L}{2}\Expect\left[\sum_{t=\tau_{4i-3,T}}^{\tau_{4i-2,T}-1}\epsilon_{t}^{m+1}\|g_{t}\|^{2}\right]\notag\\&+ \frac{1}{b-a}\Expect\left[\sum_{t=1}^{\tau_{4i-2,T}-1}\epsilon_{t}^{{m-1}}M^{2}_{t}-\sum_{t=1}^{\tau_{4i-3,T}-1}\epsilon_{t}^{{m-1}}M^{2}_{t}\right]\notag\\&+a\Expect[\epsilon_{\tau_{4i-3,T}-1}^{\frac{m-1}{2}}-\epsilon_{\tau_{4i-2,T}}^{\frac{m-1}{2}}]+\epsilon_{1}^{\frac{m-1}{2}}\Expect\left[\overline{\Delta}_{\tau_{4i-3,T}-1,(b-a)/2}\right]\notag\\&=\frac{L}{2}\Expect\left[\sum_{t=\tau_{4i-3,T}}^{\tau_{4i-2,T}-1}\epsilon_{t}^{m+1}\|g_{t}\|^{2}\right]+ \frac{1}{b-a}\Expect\left[\sum_{t=\tau_{4i-3,T}}^{\tau_{4i-2,T}-1}\epsilon_{t}^{{m-1}}M^{2}_{t}\right]\notag\\&+a\Expect[\epsilon_{\tau_{4i-3,T}-1}^{\frac{m-1}{2}}-\epsilon_{\tau_{4i-2,T}}^{\frac{m-1}{2}}]+\epsilon_{1}^{\frac{m-1}{2}}\Expect\left[\overline{\Delta}_{\tau_{4i-3,T}-1,(b-a)/2}\right].
\end{align}
Now let us explain step $(\text{i}).$ Since it has been shown earlier that $\tau_{n,T}-1$ is also a stopping time, we know from \emph{Doob's stopping theorem} that the following stopped process
\[\left(\sum_{t=1}^{\tau_{n,T}-1}\epsilon_{t}^{\frac{m-1}{2}}M_{t}, \mathscr{F}_{\tau_{n,T}-1}\right),\]
is still a martingale. Thus, we can easily derive:
\begin{align*}
&\frac{1}{b-a}\Expect\left[\sum_{t=1}^{\tau_{4i-2,T}-1}\epsilon_{t}^{\frac{m-1}{2}}M_{t}-\sum_{t=1}^{\tau_{4i-3,T}-1}\epsilon_{t}^{\frac{m-1}{2}}M_{t}\right]^{2}\\&=\frac{1}{b-a}\Expect\left[\sum_{t=1}^{\tau_{4i-2,T}-1}\epsilon_{t}^{\frac{m-1}{2}}M_{t}\right]^{2}-\frac{1}{b-a}\Expect\left[\sum_{t=1}^{\tau_{4i-3,T}-1}\epsilon_{t}^{\frac{m-1}{2}}M_{t}\right]^{2}\\&=\frac{1}{b-a}\Expect\left[\sum_{t=1}^{\tau_{4i-2,T}-1}\epsilon_{t}^{{m-1}}M^{2}_{t}\right]-\frac{1}{b-a}\Expect\left[\sum_{t=1}^{\tau_{4i-3,T}-1}\epsilon_{t}^{{m-1}}M^{2}_{t}\right].
\end{align*}
Substituting Eq \ref{sgd_-500} back into Eq. \ref{sgd_-101}, we obtain:
\begin{align*}
\mathbb{E} \Bigg[
    \mathbb{I}^{(i)} \sum_{t=\tau_{4i-2,T}}^{\tau_{4i-1,T}-1}
    \epsilon_t^m \| \nabla f(\theta_t) \|^2
\Bigg]
&\le \frac{3bL}{b - a} 
\mathbb{E} \left[
    \sum_{t=\tau_{4i-3,T}}^{\tau_{4i-2,T}-1} 
    \epsilon_t^{m+1} \| g_t \|^2
\right] \\
&\quad + \frac{6b}{(b - a)^2}
\mathbb{E} \left[
    \sum_{t=\tau_{4i-3,T}}^{\tau_{4i-2,T}-1} 
    \epsilon_t^{m-1} M_t^2
\right] \\
&\quad + \frac{6ab}{b - a} 
\mathbb{E} \left[
    \epsilon_{\tau_{4i-3,T}-1}^{\frac{m-1}{2}} 
    - \epsilon_{\tau_{4i-2,T}}^{\frac{m-1}{2}}
\right] \\
&\quad + \epsilon_1^{\frac{m-1}{2}} 
\mathbb{E} \left[
    \overline{\Delta}_{\tau_{4i-3,T}-1, \frac{b - a}{2}}
\right] \\
&\quad + \epsilon_1^{m-1} 
\mathbb{E} \left[
    \overline{\Delta}_{\tau_{4i-2,T}-1, \frac{b}{2}}
\right] \\
&\quad + \frac{L}{2}
\mathbb{E} \left[
    \mathbb{I}^{(i)} \sum_{t=\tau_{4i-2,T}}^{\tau_{4i-1,T}-1} 
    \epsilon_t^{m+1} \| g_t \|^2
\right].
\end{align*}

We sum both sides of the above inequality over the index \(i\), yielding:
\begin{align}\label{sgd_0123123}
\sum_{i=1}^{+\infty}\Expect\Bigg[\I^{(i)}\sum_{t=\tau_{4i-2,T}}^{\tau_{4i-1,T}-1}\epsilon^{m}_{t}\|\nabla f(\theta_{t})\|^{2}\Bigg]&{\le}\left(\frac{3bL}{b-a}+\frac{L}{2}\right)\sum_{i=1}^{+\infty}\underbrace{\Expect\left[\sum_{t=\tau_{4i-3,T}}^{\tau_{4i-2,T}-1}\epsilon_{t}^{m+1}\|g_{t}\|^{2}\right]}_{\Gamma_{i,1}}\notag\\&+\frac{6b}{(b-a)^{2}}\sum_{i=1}^{+\infty}\underbrace{\Expect\left[\sum_{t=\tau_{4i-3,T}}^{\tau_{4i-2,T}-1}\epsilon_{t}^{{m-1}}M^{2}_{t}\right]}_{\Gamma_{i,2}}\notag\\&+\frac{6ab}{b-a}\epsilon_{1}^{\frac{m-1}{2}}+(\epsilon_{1}^{\frac{m-1}{2}}+\epsilon_{1}^{m-1})C_{\frac{b-a}{2}}.
\end{align}
In above inequality, the final term \((\epsilon_{1}^{\frac{m-1}{2}}+\epsilon_{1}^{m-1})C_{\frac{b-a}{2}}\) is derived as follows:
\begin{align*}
&\quad 
\epsilon_1^{\frac{m-1}{2}} 
\sum_{i=1}^{+\infty} 
\overline{\Delta}_{\tau_{4i-3,T}-1, \frac{b - a}{2}} 
+ \epsilon_1^{m-1} 
\sum_{i=1}^{+\infty} 
\overline{\Delta}_{\tau_{4i-2,T}-1, \frac{b}{2}} \\
&\le 
\left( \epsilon_1^{\frac{m-1}{2}} + \epsilon_1^{m-1} \right)
\sum_{i=1}^{+\infty} 
\overline{\Delta}_{\tau_{4i-3,T}-1, \frac{b - a}{2}} \\
&\overset{\text{Lemma~\ref{lem_1123}}}{\le} 
\left( \epsilon_1^{\frac{m-1}{2}} + \epsilon_1^{m-1} \right)
C_{\frac{b - a}{2}}.
\end{align*}
Next, we simplify \( \Gamma_{i,1} \) and \( \Gamma_{i,2} \) in Eq. \ref{sgd_0123123}.

For \( \Gamma_{i,1} \), we have:
\begin{align}\label{sgd__2}
 \Gamma_{i,1}&\mathop{=}^{\text{\emph{Doob's Stopped Theorem}}} \Expect\left[\sum_{t=\tau_{4i-3,T}}^{\tau_{4i-2,T}-1}\epsilon_{t}^{m+1}\Expect\left[\|g_{t}\|^{2}|\mathscr{F}_{t-1}\right]\right]\notag\\&\mathop{\le}^{\text{Assumption \ref{assump:stochastic_gradient} Item $(b)$}}  G\cdot\Expect\left[\sum_{t=\tau_{4i-3,T}}^{\tau_{4i-2,T}-1}\epsilon_{t}^{m+1}\left(\|\nabla f(\theta_{t})\|^{2}+1\right)\right]\notag\\&\mathop{\le}^{(\text{i})}\frac{2G}{\delta^{2}_{a,b}}\Expect\left[\sum_{t=\tau_{4i-3,T}}^{\tau_{4i-2,T}-1}\epsilon_{t}^{m+1}\|\nabla f(\theta_{t})\|^{2}\right].
\end{align}
The detailed derivation for step \((\text{i})\) is as follows:

\begin{align*}
\mathbb{E} \left[
    \sum_{t=\tau_{4i-3,T}}^{\tau_{4i-2,T}-1}
    \epsilon_t^{m+1} \left( \|\nabla f(\theta_t)\|^2 + 1 \right)
\right]&= \mathbb{E} \left[
    \mathbb{I}_{\tau_{4i-3,T} = \tau_{4i-2,T}}
    \sum_{t=\tau_{4i-3,T}}^{\tau_{4i-2,T}-1}
    \epsilon_t^{m+1} \left( \|\nabla f(\theta_t)\|^2 + 1 \right)
\right] \\
 +& \mathbb{E} \left[
    \mathbb{I}_{\tau_{4i-3,T} < \tau_{4i-2,T}}
    \sum_{t=\tau_{4i-3,T}}^{\tau_{4i-2,T}-1}
    \epsilon_t^{m+1} \left( \|\nabla f(\theta_t)\|^2 + 1 \right)
\right] \\
=& 0 + \mathbb{E} \left[
    \mathbb{I}_{\tau_{4i-3,T} < \tau_{4i-2,T}}
    \sum_{t=\tau_{4i-3,T}}^{\tau_{4i-2,T}-1}
    \epsilon_t^{m+1} \left( \|\nabla f(\theta_t)\|^2 + 1 \right)
\right] \\
\overset{(*)}{\le}&
\left(1 + \frac{1}{\delta_{a,b}^2} \right)
\mathbb{E} \left[
    \mathbb{I}_{\tau_{4i-3,T} < \tau_{4i-2,T}}
    \sum_{t=\tau_{4i-3,T}}^{\tau_{4i-2,T}-1}
    \epsilon_t^{m+1} \|\nabla f(\theta_t)\|^2
\right] \\
\le& \left(1 + \frac{1}{\delta_{a,b}^2} \right)
\mathbb{E} \left[
    \sum_{t=\tau_{4i-3,T}}^{\tau_{4i-2,T}-1}
    \epsilon_t^{m+1} \|\nabla f(\theta_t)\|^2
\right].
\end{align*}

In the above derivation, step \((*)\) requires noting that when \( \tau_{4i-3,T}<\tau_{4i-2,T}\), we simultaneously have \(\tau_{4i-3} < T \) and \( \tau_{4i-3} < \tau_{4i-2} \). This implies that for any \( t \in [\tau_{4i-3,T},\tau_{4i-2,T}-1] \), it holds that \( \|\nabla f(\theta_{t})\| \ge \delta_{a,b}.\)

For \( \Gamma_{i,2} \), we have:
\begin{align}\label{sgd__3}
 \Gamma_{i,2}&\mathop{=}^{\text{\emph{Doob's Stopped Theorem}}} \Expect\left[\sum_{t=\tau_{4i-3,T}}^{\tau_{4i-2,T}-1}\epsilon_{t}^{{m-1}}\Expect\left[M^{2}_{t}|\mathscr{F}_{t-1}\right]\right]\notag\\&= \Expect\left[\I_{\tau_{4i-3,T}=\tau_{4i-2,T}}\sum_{t=\tau_{4i-3,T}}^{\tau_{4i-2,T}-1}\epsilon_{t}^{{m-1}}\Expect\left[M^{2}_{t}|\mathscr{F}_{t-1}\right]\right]\notag\\&\quad+ \Expect\left[\I_{\tau_{4i-3,T}<\tau_{4i-2,T}}\sum_{t=\tau_{4i-3,T}}^{\tau_{4i-2,T}-1}\epsilon_{t}^{{m-1}}\Expect\left[M^{2}_{t}|\mathscr{F}_{t-1}\right]\right]\notag\\&=0+\Expect\left[\I_{\tau_{4i-3,T}<\tau_{4i-2,T}}\sum_{t=\tau_{4i-3,T}}^{\tau_{4i-2,T}-1}\epsilon_{t}^{{m-1}}\Expect\left[M^{2}_{t}|\mathscr{F}_{t-1}\right]\right]\notag\\&=\Expect\left[\I_{\tau_{4i-3,T}<\tau_{4i-2,T}}\sum_{t=\tau_{4i-3,T}}^{\tau_{4i-2,T}-1}\epsilon_{t}^{{m-1}}\epsilon^{2}_{t}\|\nabla f(\theta_{t})\|^{2}\Expect\left[\|\nabla f(\theta_{t})-g_{t}\|^{2}|\mathscr{F}_{t-1}\right]\right]\notag\\&\mathop{\le}^{(\text{i})}2Lc\Expect\left[\I_{\tau_{4i-3,T}<\tau_{4i-2,T}}\sum_{t=\tau_{4i-3,T}}^{\tau_{4i-2,T}-1}\epsilon_{t}^{{m+1}}\Expect\left[\|g_{t}\|^{2}|\mathscr{F}_{t-1}\right]\right]\notag\\&\mathop{\le}^{\text{Assumption \ref{assump:stochastic_gradient} Item $(b)$}}2LcG\Expect\left[\I_{\tau_{4i-3,T}<\tau_{4i-2,T}}\sum_{t=\tau_{4i-3,T}}^{\tau_{4i-2,T}-1}\epsilon_{t}^{{m+1}}\left(\|\nabla f(\theta_{t})\|^{2}+1\right)\right]\notag\\&\mathop{\le}^{(\text{ii})}2LcG\left(1+\frac{1}{\delta_{a,b}^{2}}\right)\Expect\left[\I_{\tau_{4i-3,T}<\tau_{4i-2,T}}\sum_{t=\tau_{4i-3,T}}^{\tau_{4i-2,T}-1}\epsilon_{t}^{{m+1}}\|\nabla f(\theta_{t})\|^{2}\right]\notag\\&\le 2LcG\left(1+\frac{1}{\delta_{a,b}^{2}}\right)\Expect\left[\sum_{t=\tau_{4i-3,T}}^{\tau_{4i-2,T}-1}\epsilon_{t}^{{m+1}}\|\nabla f(\theta_{t})\|^{2}\right] .
\end{align}
The detailed derivation of step \((\text{i})\) is as follows:
\begin{align*}
 &\quad\I_{\tau_{4i-3,T}<\tau_{4i-2,T}}\sum_{t=\tau_{4i-3,T}}^{\tau_{4i-2,T}-1}\epsilon_{t}^{{m-1}}\epsilon^{2}_{t}\|\nabla f(\theta_{t})\|^{2}\Expect\left[\|\nabla f(\theta_{t})-g_{t}\|^{2}|\mathscr{F}_{t-1}\right]\\&\le  \I_{\tau_{4i-3,T}<\tau_{4i-2,T}}2L\sum_{t=\tau_{4i-3,T}}^{\tau_{4i-2,T}-1}\epsilon_{t}^{{m+1}}(f(\theta_{t})-f^*)\Expect\left[\|\nabla f(\theta_{t})-g_{t}\|^{2}|\mathscr{F}_{t-1}\right] ,
\end{align*}
We need to note that  when \( \tau_{4i-3,T}<\tau_{4i-2,T}\), we simultaneously have \(\tau_{4i-3} < T \) and \( \tau_{4i-3} < \tau_{4i-2} \). This implies that for any \( t \in [\tau_{4i-3,T},\tau_{4i-2,T}-1] \), it holds that \( f(\theta_{t})-f^*\le b.\) Then we have:
\begin{align*}
 &\quad  \I_{\tau_{4i-3,T}<\tau_{4i-2,T}}2L\sum_{t=\tau_{4i-3,T}}^{\tau_{4i-2,T}-1}\epsilon_{t}^{{m+1}}(f(\theta_{t})-f^*)\Expect\left[\|\nabla f(\theta_{t})-g_{t}\|^{2}|\mathscr{F}_{t-1}\right]\notag\\&\le \I_{\tau_{4i-3,T}<\tau_{4i-2,T}}2Lb\sum_{t=\tau_{4i-3,T}}^{\tau_{4i-2,T}-1}\epsilon_{t}^{{m+1}}\Expect\left[\|\nabla f(\theta_{t})-g_{t}\|^{2}|\mathscr{F}_{t-1}\right]\notag\\&\le \I_{\tau_{4i-3,T}<\tau_{4i-2,T}}2Lb\sum_{t=\tau_{4i-3,T}}^{\tau_{4i-2,T}-1}\epsilon_{t}^{{m+1}}\Expect\left[\|g_{t}\|^{2}|\mathscr{F}_{t-1}\right] .
\end{align*}
For step \((\text{ii})\), we similarly need to note that when \( \tau_{4i-3,T}<\tau_{4i-2,T}\), we simultaneously have \(\tau_{4i-3} < T \) and \( \tau_{4i-3} < \tau_{4i-2} \). This implies that for any \( t \in [\tau_{4i-3,T},\tau_{4i-2,T}-1] \), it holds that \( \|\nabla f(\theta_{t})\|^{2} \ge \delta^{2}_{a,b}.\) 

Substituting the estimates of \( \Gamma_{i,1} \) and \( \Gamma_{i,2} \) from Eq. \ref{sgd__2} and Eq. \ref{sgd__3} back into Eq. \ref{sgd_0123123}, we obtain:
\begin{align}\label{dsadsak}
&\quad\sum_{i=1}^{+\infty}\Expect\Bigg[\I^{(i)}\sum_{t=\tau_{4i-2,T}}^{\tau_{4i-1,T}-1}\epsilon^{m}_{t}\|\nabla f(\theta_{t})\|^{2}\Bigg]\notag\\&{\le}\underbrace{\left(\left(\frac{3bL}{b-a}+\frac{L}{2}\right)\frac{2G}{\delta_{a,b}^{2}}+\frac{12Lb^{2}G}{(b-a)^{2}}\right)\left(1+\frac{1}{\delta_{a,b}^{2}}\right)}_{C_{1}(a,b)}\sum_{i=1}^{+\infty}\Expect\left[\sum_{t=\tau_{4i-3,T}}^{\tau_{4i-2,T}-1}\epsilon_{t}^{{m+1}}\|\nabla f(\theta_{t})\|^{2}\right] \notag\\&\quad+\underbrace{\frac{6ab}{b-a}\epsilon_{1}^{\frac{m-1}{2}}(\epsilon_{1}^{\frac{m-1}{2}}+\epsilon_{1}^{m-1})C_{\frac{b-a}{2}}}_{C_{2}(m,a,b)}.
\end{align}
On the other hand, we can easily obtain the following two equations:  
\[
\sum_{i=1}^{+\infty}\Expect\Bigg[\I^{(i)}\sum_{t=\tau_{4i-2,T}}^{\tau_{4i-1,T}-1}\epsilon^{m}_{t}\|\nabla f(\theta_{t})\|^{2}\Bigg]=\big[\nabla f(\theta_{t})\big]_{T,b,c}^{m}
\]  
and  
\[
\sum_{i=1}^{+\infty}\Expect\left[\sum_{t=\tau_{4i-3,T}}^{\tau_{4i-2,T}-1}\epsilon_{t}^{{m+1}}\|\nabla f(\theta_{t})\|^{2}\right]=\big[\nabla f(\theta_{t})\big]_{T,a,b}^{m+1}.
\]
Based on the two equalities above, we can finally obtain:
\begin{align*}
\big[\nabla f(\theta_{t})\big]_{T,b,c}^{m}&\le C_{1}(a,b) \big[\nabla f(\theta_{t})\big]_{T,a,b}^{m+1}+ C_{2}(m,a,b).
\end{align*}

\end{proof}

\end{document}

%% file: math_commands.tex
\usepackage{amsmath,amsfonts,bm}

\def\eqref#1{equation~\ref{#1}}

\def\1{\bm{1}}

\DeclareMathAlphabet{\mathsfit}{\encodingdefault}{\sfdefault}{m}{sl}
\SetMathAlphabet{\mathsfit}{bold}{\encodingdefault}{\sfdefault}{bx}{n}

